\documentclass{article}
\usepackage[letterpaper, left=1in, right=1in, top=1in, bottom=1in]{geometry}

\usepackage{natbib}
\bibpunct{(}{)}{;}{a}{,}{,}
\usepackage[colorlinks,citecolor=blue!70!black,linkcolor=blue!70!black,
urlcolor=blue!70!black,breaklinks=true]{hyperref}

\usepackage{parskip}

\usepackage[usenames, dvipsnames]{xcolor}
\usepackage{microtype}
\usepackage{booktabs} 
\usepackage{amsmath}
\usepackage{dsfont} 
\usepackage{amssymb}
\usepackage{xspace}
\usepackage{url}            
\usepackage{algorithm}
\usepackage[noend]{algorithmic}
\usepackage{amsthm}
\usepackage{listings}
\usepackage{bm}
\usepackage{bbm}
\usepackage{enumitem}
\usepackage{nicefrac}       
\usepackage{mathrsfs} 
\usepackage{inconsolata}
\usepackage[subrefformat=parens,labelformat=parens]{subfig}
\usepackage{wrapfig}
\usepackage{graphicx}
\usepackage{bigints}
\usepackage{transparent}
\usepackage{comment}
\usepackage[nameinlink,capitalize]{cleveref}
\usepackage{thmtools}
\usepackage{thm-restate}
\usepackage{nicematrix}
\usepackage{arydshln}
\usepackage{fancyhdr}
\usepackage{mathtools}
\usepackage[scaled=.88]{helvet}

\usepackage{breakcites}

\newcommand{\pref}[1]{\cref{#1}}
\newcommand{\pfref}[1]{Proof of \pref{#1}}

\crefformat{equation}{#2Eq. (#1)#3}
\Crefformat{equation}{#2Eq. (#1)#3}

\Crefformat{figure}{#2Figure #1#3}
\Crefformat{assumption}{#2Assumption #1#3}

\usepackage{crossreftools}
\makeatletter
\renewcommand{\crtaddlabeltotoc}[1]{%
  \ifcrtfinal
  \else
  \addcontentsline{\crt@listoflabelsfileextension}{\crt@listoflabelsstructurelevel}{\protect\numberline{\string\crtrefnumber{#1}}#1}%
  \fi
}
\makeatother

\newtheorem{theorem}{Theorem}
\newtheorem{lemma}{Lemma}

\newtheorem{definition}{Definition}
\newtheorem{assumption}{Assumption}

\newtheorem{example}{Example}

\DeclarePairedDelimiter\norm{\lVert}{\rVert}\DeclarePairedDelimiter\ang{\langle}{\rangle}\DeclarePairedDelimiter\curly{\{}{\}}
\DeclarePairedDelimiter\paren{(}{)}
\DeclarePairedDelimiter\sq{[}{]}

\makeatletter
\let\save@mathaccent\mathaccent
\newcommand*\if@single[3]{\setbox0\hbox{${\mathaccent"0362{#1}}^H$}\setbox2\hbox{${\mathaccent"0362{\kern0pt#1}}^H$}\ifdim\ht0=\ht2 #3\else #2\fi
  }
\newcommand*\rel@kern[1]{\kern#1\dimexpr\macc@kerna}
\newcommand*\widebar[1]{\@ifnextchar^{{\wide@bar{#1}{0}}}{\wide@bar{#1}{1}}}
\newcommand*\wide@bar[2]{\if@single{#1}{\wide@bar@{#1}{#2}{1}}{\wide@bar@{#1}{#2}{2}}}
\newcommand*\wide@bar@[3]{\begingroup
  \def\mathaccent##1##2{\let\mathaccent\save@mathaccent
\if#32 \let\macc@nucleus\first@char \fi
\setbox\z@\hbox{$\macc@style{\macc@nucleus}_{}$}\setbox\tw@\hbox{$\macc@style{\macc@nucleus}{}_{}$}\dimen@\wd\tw@
    \advance\dimen@-\wd\z@
\divide\dimen@ 3
    \@tempdima\wd\tw@
    \advance\@tempdima-\scriptspace
\divide\@tempdima 10
    \advance\dimen@-\@tempdima
\ifdim\dimen@>\z@ \dimen@0pt\fi
\rel@kern{0.6}\kern-\dimen@
    \if#31
      \overline{\rel@kern{-0.6}\kern\dimen@\macc@nucleus\rel@kern{0.4}\kern\dimen@}\advance\dimen@0.4\dimexpr\macc@kerna
\let\final@kern#2\ifdim\dimen@<\z@ \let\final@kern1\fi
      \if\final@kern1 \kern-\dimen@\fi
    \else
      \overline{\rel@kern{-0.6}\kern\dimen@#1}\fi
  }\macc@depth\@ne
  \let\math@bgroup\@empty \let\math@egroup\macc@set@skewchar
  \mathsurround\z@ \frozen@everymath{\mathgroup\macc@group\relax}\macc@set@skewchar\relax
  \let\mathaccentV\macc@nested@a
\if#31
    \macc@nested@a\relax111{#1}\else
\def\gobble@till@marker##1\endmarker{}\futurelet\first@char\gobble@till@marker#1\endmarker
    \ifcat\noexpand\first@char A\else
      \def\first@char{}\fi
    \macc@nested@a\relax111{\first@char}\fi
  \endgroup
}
\makeatother

\DeclareMathOperator*{\argmax}{arg\,max}

\DeclareMathOperator{\spn}{span}

\DeclareMathOperator{\sgn}{sign}

\DeclareMathOperator{\poly}{poly}
\DeclareMathOperator{\supp}{supp}

\DeclareMathOperator{\conv}{{conv}}
\DeclareMathOperator{\opt}{{opt}}

\def\E{{\mathbb E}}

\def\R{{\mathbb R}}

\def\ddefloop#1{\ifx\ddefloop#1\else\ddef{#1}\expandafter\ddefloop\fi}
\def\ddef#1{\expandafter\def\csname bb#1\endcsname{\ensuremath{\mathbb{#1}}}}
\ddefloop ABCDEFGHIJKLMNOPQRSTUVWXYZ\ddefloop
\def\ddefloop#1{\ifx\ddefloop#1\else\ddef{#1}\expandafter\ddefloop\fi}
\def\ddef#1{\expandafter\def\csname b#1\endcsname{\ensuremath{\mathbf{#1}}}}
\ddefloop ABCDEFGHIJKLMNOPQRSTUVWXYZ\ddefloop
\def\ddef#1{\expandafter\def\csname sf#1\endcsname{\ensuremath{\mathsf{#1}}}}
\ddefloop ABCDEFGHIJKLMNOPQRSTUVWXYZ\ddefloop
\def\ddef#1{\expandafter\def\csname c#1\endcsname{\ensuremath{\mathcal{#1}}}}
\ddefloop ABCDEFGHIJKLMNOPQRSTUVWXYZ\ddefloop
\def\ddef#1{\expandafter\def\csname h#1\endcsname{\ensuremath{\widehat{#1}}}}
\ddefloop ABCDEFGHIJKLMNOPQRSTUVWXYZ\ddefloop
\def\ddef#1{\expandafter\def\csname hc#1\endcsname{\ensuremath{\widehat{\mathcal{#1}}}}}
\ddefloop ABCDEFGHIJKLMNOPQRSTUVWXYZ\ddefloop
\def\ddef#1{\expandafter\def\csname t#1\endcsname{\ensuremath{\widetilde{#1}}}}
\ddefloop ABCDEFGHIJKLMNOPQRSTUVWXYZ\ddefloop
\def\ddef#1{\expandafter\def\csname tc#1\endcsname{\ensuremath{\widetilde{\mathcal{#1}}}}}
\ddefloop ABCDEFGHIJKLMNOPQRSTUVWXYZ\ddefloop
\def\ddefloop#1{\ifx\ddefloop#1\else\ddef{#1}\expandafter\ddefloop\fi}
\def\ddef#1{\expandafter\def\csname scr#1\endcsname{\ensuremath{\mathscr{#1}}}}
\ddefloop ABCDEFGHIJKLMNOPQRSTUVWXYZ\ddefloop

\newcommand{\rwalgtext}{\textsf{ReweightedSpanner}\xspace}

\newcommand{\ahat}{\wh{a}}

\newcommand{\phibar}{\wb{\phi}}

\newcommand{\abar}{\wb{a}}

\newcommand{\vepsbar}{\wb{\veps}}

\newcommand{\acheck}{\check{a}}

\newcommand{\approxleq}{\lesssim}

\newcommand{\bigoh}{O}
\newcommand{\bigoht}{\wt{O}}
\newcommand{\bigom}{\Omega}
\newcommand{\bigomt}{\wt{\Omega}}

\newcommand{\polylog}{\mathrm{polylog}}

\newcommand{\indic}{\mathbb{I}}

\newcommand{\algcommentlight}[1]{\textcolor{blue!70!black}{\transparent{0.5}\small{\texttt{\textbf{//\hspace{2pt}#1}}}}}

\DeclarePairedDelimiter{\abs}{\lvert}{\rvert} 
\DeclarePairedDelimiter{\brk}{[}{]}
\DeclarePairedDelimiter{\crl}{\{}{\}}
\DeclarePairedDelimiter{\prn}{(}{)}
\DeclarePairedDelimiter{\nrm}{\|}{\|}
\DeclarePairedDelimiter{\tri}{\langle}{\rangle}

\DeclarePairedDelimiter{\ceil}{\lceil}{\rceil}

\newcommand{\wt}[1]{\widetilde{#1}}
\newcommand{\wh}[1]{\widehat{#1}}
\newcommand{\wb}[1]{\widebar{#1}}

\newcommand{\eps}{\epsilon}
\newcommand{\veps}{\varepsilon}

\newcommand{\ldef}{\vcentcolon=}
\newcommand{\rdef}{=\vcentcolon}

\newcommand{\RegCB}{\mathrm{\mathbf{Reg}}_{\mathsf{CB}}}
\newcommand{\RegSq}{\mathrm{\mathbf{Reg}}_{\mathsf{Sq}}}

\newcommand{\AlgSq}{\mathrm{\mathbf{Alg}}_{\mathsf{Sq}}}
\newcommand{\AlgOpt}{\mathrm{\mathbf{Alg}}_{\mathsf{Opt}}}
\newcommand{\dec}{\mathsf{dec}}
\newcommand{\SqAlg}{\AlgSq}

\newcommand{\regcb}{\RegCB}
\newcommand{\regsq}{\RegSq}

\newcommand{\sqalgtext}{$\SqAlg$\xspace}

\newcommand{\optalg}{\AlgOpt}
\newcommand{\optalgtext}{$\optalg$\xspace}

\newcommand{\Topt}{\cT_{\mathsf{Opt}}}
\newcommand{\Tsq}{\cT_{\mathsf{Sq}}}

\newcommand{\Msq}{\cM_{\mathsf{Sq}}}
\newcommand{\Mopt}{\cM_{\mathsf{Opt}}}

\newcommand{\dectext}{Decision-Estimation Coefficient\xspace}

\newcommand{\des}{q^{\opt}}
\newcommand{\dess}{{\check q}^{\opt}}

\newcommand{\squarecb}{\textsf{SquareCB}\xspace}
\newcommand{\greedy}{$\epsilon$-\textsf{Greedy}\xspace}
\newcommand{\spannerGreedy}{\textsf{SpannerGreedy}\xspace}
\newcommand{\spannerIGW}{\textsf{SpannerIGW}\xspace}
\newcommand{\argmaxIGW}{\textsf{IGW-ArgMax}\xspace}
\newcommand{\ball}{\mathsf{Ball}\xspace}

\newcommand{\oneshotwiki}{\textsf{oneshotwiki}\xspace}
\newcommand{\amazon}{\textsf{amazon}\xspace}
\newcommand{\sentencetransformers}{\textsf{sentence transformers}\xspace}
\newcommand{\supervised}{\textsf{supervised}\xspace}

\newcommand{\mainalg}{\spannerIGW}
\newcommand{\greedyalg}{\spannerGreedy}

\newcommand{\astar}{a^{\star}}

\newcommand{\fstar}{f^{\star}}
\newcommand{\gstar}{g^{\star}}

\newcommand{\fhat}{\wh{f}}
\newcommand{\ghat}{\wh{g}}

\newcommand{\unif}{\mathrm{unif}}

\renewcommand{\cite}[1]{\citet{#1}}

\let\oldparagraph\paragraph
\renewcommand{\paragraph}[1]{\oldparagraph{#1.}}

\renewcommand{\epsilon}{\varepsilon}

\title{Contextual Bandits with Large Action Spaces: Made Practical}
\date{}
\author{
Yinglun Zhu$^{\dagger}$\\
{\small\texttt{yinglun@cs.wisc.edu}}
\and
Dylan J. Foster$^{\ddagger}$\\
{\small\texttt{dylanfoster@microsoft.com}}
\and
John Langford$^{\ddagger}$\\
{\small\texttt{jcl@microsoft.com}}
\and
Paul Mineiro$^{\ddagger}$\\
{\small\texttt{pmineiro@microsoft.com}}
\and
~\\
{\normalsize University of Wisconsin-Madison$^{\dagger}$ \quad \quad  Microsoft Research, NYC$^{\ddagger}$}
}

\begin{document}

\maketitle

\begin{abstract}
A central problem in sequential decision making is to develop algorithms that are practical and computationally efficient, yet support the use of flexible, general-purpose models. Focusing on the contextual bandit problem, recent progress provides provably efficient algorithms with strong empirical performance when the number of possible alternatives (``actions'') is small, but guarantees for decision making in large, continuous action spaces have remained elusive, leading to a significant gap between theory and practice. We present the first efficient, general-purpose algorithm for contextual bandits with continuous, linearly structured action spaces. Our algorithm makes use of computational oracles for (i) supervised learning, and (ii) optimization over the action space, and achieves sample complexity, runtime, and memory independent of the size of the action space. In addition, it is simple and practical. We perform a large-scale empirical evaluation, and show that our approach typically enjoys superior performance and efficiency compared to standard baselines.

 \end{abstract}

\section{Introduction}
\label{sec:intro}

We consider the design of practical, theoretically motivated algorithms for sequential decision making with contextual information, better known as the \emph{contextual bandit problem}. Here, a learning agent repeatedly receives a \emph{context} (e.g., a user's profile), selects an \emph{action} (e.g., a news article to display), and receives a \emph{reward} (e.g., whether the article was clicked). Contextual bandits are a useful model for decision making in unknown environments in which both exploration and generalization are required, but pose significant algorithm design challenges beyond classical supervised learning. Recent years have seen development on two fronts:
On the theoretical side, extensive research into finite-action contextual bandits has resulted in practical, provably efficient algorithms capable of supporting flexible, general-purpose models \citep{langford2007epoch,agarwal2014taming,foster2020beyond,simchi2021bypassing,foster2021efficient}.
Empirically, contextual bandits have been widely deployed in practice for online personalization and recommendation tasks \citep{li2010contextual,agarwal2016making,tewari2017ads,cai2021bandit}, leveraging the availability of high-quality action slates (e.g., subsets of candidate articles selected by an editor).

The developments above critically rely on the existence of a small number of possible decisions or alternatives. However, many applications demand the ability to make contextual decisions in large, potentially continuous spaces, where actions might correspond to images in a database or high-dimensional embeddings of rich documents such as webpages. Contextual bandits in large (e.g., million-action) settings remains a major challenge---both statistically and computationally---and constitutes a substantial gap between theory and practice. In particular:

\begin{itemize}
\item Existing \emph{general-purpose} algorithms \citep{langford2007epoch,agarwal2014taming,foster2020beyond,simchi2021bypassing,foster2021efficient} allow for the use of flexible models (e.g., neural networks, forests, or kernels) to facilitate generalization across contexts, but have sample complexity and computational requirements linear in the number of actions. These approaches can degrade in performance under benign operations such as duplicating actions.

\item While certain recent approaches extend the general-purpose methods above to accommodate large action spaces, they either require sample complexity exponential in action dimension~\citep{krishnamurthy2020contextual}, or require additional distributional assumptions \citep{sen2021top}.
\item Various results efficiently handle large or continuous action spaces \citep{dani2008stochastic, jun2017scalable, yang2021linear} with specific types of function approximation, but do not accommodate general-purpose models.
\end{itemize}
As a result of these algorithmic limitations, empirical aspects of contextual decision making in large action spaces have remained relatively unexplored compared to the small-action regime \citep{bietti2021contextual}, with little in the way of readily deployable out-of-the-box solutions.

\paragraph{Contributions}
We provide the first efficient algorithms for contextual bandits with continuous, linearly structured action spaces and general function approximation. Following \cite{chernozhukov2019semi,xu2020upper,foster2020adapting}, we adopt a modeling approach, and assume rewards for each context-action pair $(x,a)$ are structured as
\begin{equation}
  \label{eq:linear}
  \fstar(x,a)=\tri*{\phi(x,a),\gstar(x)}.
\end{equation}
Here $\phi(x,a)\in\bbR^{d}$ is a known context-action embedding (or feature map) and $\gstar\in\cG$ is a context embedding to be learned online, which belongs to an arbitrary, user-specified function class $\cG$. Our algorithm, \mainalg, is computationally efficient (in particular, the runtime and memory are \emph{independent} of the number of actions) whenever the user has access to (i) an \emph{online regression oracle} for supervised learning over the reward function class, and (ii) an \emph{action optimization oracle} capable of solving problems of the form
\begin{align*}
\argmax_{a\in\cA}\tri*{\phi(x,a),\theta}
\end{align*}
for any $\theta\in\bbR^{d}$. The former oracle follows prior approaches to finite-action contextual bandits \citep{foster2020beyond,simchi2021bypassing,foster2021efficient}, while the latter generalizes efficient approaches to (non-contextual) linear bandits \citep{mcmahan2004online,dani2008stochastic,bubeck2012towards,hazan2016volumetric}. We provide a regret bound for \mainalg which scales as $\sqrt{\poly(d)\cdot T}$, and---like the computational complexity---is independent of the number of actions. Beyond these results, we provide a particularly practical variant of \mainalg (\greedyalg), which enjoys even faster runtime at the cost of slightly worse ($\poly(d)\cdot{}T^{2/3}$-type) regret.

\paragraph{Our techniques}
On the technical side, we show how to
\emph{efficiently} combine the inverse gap weighting technique \citep{abe1999associative,foster2020beyond} previously used in the finite-action setting with optimal design-based approaches for exploration with linearly structured actions. This offers a computational improvement upon the results of \cite{xu2020upper,foster2020adapting}, which provide algorithms with $\sqrt{\poly(d)\cdot{}T}$-regret for the setting we consider, but require enumeration over the action space. Conceptually, our results expand upon the class of problems for which minimax approaches to exploration \citep{foster2021statistical} can be made efficient.

\paragraph{Empirical performance}
As with previous approaches based on regression oracles, \mainalg is simple, practical, and well-suited to flexible, general-purpose function approximation. In extensive experiments ranging from thousands to millions of actions, we find that our methods typically enjoy superior performance compared to existing baselines. In addition, our experiments validate the statistical model in \cref{eq:linear} 
which we find to be well-suited to learning with large-scale language models \citep{devlin2019bert}.

\subsection{Organization}
This paper is organized as follows. In \pref{sec:setting}, we formally introduce our statistical model and the computational oracles upon which our algorithms are built. Subsequent sections are dedicated to our main results.
\begin{itemize}
\item As a warm-up, \pref{sec:warmup} presents a simplified algorithm, \greedyalg, which illustrates the principle of exploration over an approximate optimal design. This algorithm is practical and oracle-efficient, but has suboptimal $\poly(d)\cdot{}T^{2/3}$-type regret.
\item 
  Building on these ideas, \pref{sec:minimax} presents our main algorithm, \mainalg, which combines the idea of approximate optimal design used by \greedyalg with the inverse gap weighting method \citep{abe1999associative,foster2020beyond}, resulting in an oracle-efficient algorithm with $\sqrt{\poly(d)\cdot{}T}$-regret.
\end{itemize}
\pref{sec:experiments} presents empirical results for both
algorithms. We close with discussion of additional related work
(\pref{sec:related}) and future directions (\pref{sec:discussion}).
All proofs are deferred to the appendix.

\section{Problem Setting}

\label{sec:setting}

The contextual bandit problem proceeds over $T$ rounds. At each round $t\in\brk{T}$, the learner receives a context $x_t \in \cX$ (the \emph{context space}), selects an action $a_t\in\cA$ (the \emph{action space}), and then observes a reward $r_t(a_t)$, where $r_t:\cA\to\brk{-1,1}$ is the underlying reward function.
We assume that for each round $t$, conditioned on $x_t$, the reward $r_t$ is sampled from a (unknown) distribution $\bbP_{r_t}(\cdot\mid{}x_t)$. We allow both the contexts $x_{1},\ldots,x_T$ and the distributions $\bbP_{r_1},\ldots,\bbP_{r_T}$ to be selected in an arbitrary, potentially adaptive fashion based on the history.

\paragraph{Function approximation}

Following a standard approach to developing efficient contextual bandit methods, we take a modeling approach, and work with a user-specified class of regression functions $\cF \subseteq (\cX \times \cA \rightarrow [-1,1])$ that aims to model the underlying mean reward function. We make the following realizability assumption \citep{agarwal2012contextual, foster2018practical, foster2020beyond, simchi2021bypassing}.
\begin{assumption}[Realizability]
  \label{asm:realizability}
  There exists a regression function $f^\star \in \cF$ such that $ \E \sq{r_t(a) \mid x_t=x} = f^\star(x, a)$ for all $a\in\cA$ and $t\in\brk{T}$.
\end{assumption}
Without further assumptions, there exist function classes $\cF$ for which the regret of any algorithm must grow proportionally to $\abs{\cA}$ (e.g., \cite{agarwal2012contextual}). In order to facilitate generalization across actions and achieve sample complexity and computational complexity independent of $\abs{\cA}$, we assume that each function $f\in\cF$ is linear in a known (context-dependent) feature embedding of the action. Following \citet{xu2020upper,foster2020adapting}, we assume that $\cF$ takes the form
\begin{align*}
  \cF = \crl*{f_g \prn{x,a} = \ang{\phi \prn{x,a}, g \prn{x}}: g \in \cG  },
\end{align*}
where $\phi(x,a)\in\bbR^{d}$ is a known, context-dependent action embedding and $\cG$ is a user-specified class of context embedding functions.

This formulation assumes linearity in the action space (after featurization), but allows for nonlinear, \emph{learned} dependence on the context $x$ through the function class $\cG$, which can be taken to consist of neural networks, forests, or any other flexible function class a user chooses. For example, in news article recommendation, $\phi(x,a)=\phi(a)$ might correspond to an embedding of an article $a$ obtained using a large pre-trained language-model, while $g(x)$ might correspond to a task-dependent embedding of a user $x$, which our methods can learn online. Well-studied special cases include the linear contextual bandit setting \citep{chu2011contextual,abbasi2011improved}, which corresponds to the special case where each $g\in\cG$ has the form $g \prn{x} = \theta$ for some fixed $\theta \in \R^d$, as well as the standard finite-action contextual bandit setting, where $d=\abs{\cA}$ and $\phi(x,a)=e_a$.

We let $\gstar\in\cG$ denote the embedding for which $\fstar=f_{\gstar}$. We assume that $\sup_{x \in \cX,a \in \cA}\nrm{\phi(x,a)} \leq 1$ and $\sup_{g \in \cG, x \in \cX} \nrm{g(x)} \leq 1$. In addition, we assume that $\spn(\crl{\phi(x,a)}) = \R^d$ for all $x \in \cX$.

\paragraph{Regret}
For each regression function $f\in\cF$, let $\pi_f(x_t) \ldef \argmax_{a \in \cA} f(x_t,a)$ denote the induced policy, and define $\pi^\star \ldef \pi_{f^\star}$ as the optimal policy. We measure the performance of the learner in terms of regret: %
\begin{align*}
    \regcb(T) \ldef  \sum_{t=1}^T  r_t(\pi^\star(x_t)) - r_t(a_t).
\end{align*}

\subsection{Computational Oracles}

To derive efficient algorithms with sublinear runtime, we make use of two computational oracles: First, following \cite{foster2020beyond,simchi2021bypassing, foster2020adapting, foster2021instance}, we use an \emph{online regression oracle} for supervised learning over the reward function class $\cF$. Second, we use an \emph{action optimization oracle}, which facilitates linear optimization over the action space $\cA$ \citep{mcmahan2004online,dani2008stochastic,bubeck2012towards,hazan2016volumetric}

\paragraph{Function approximation: Regression oracles}

A fruitful approach to designing efficient contextual bandit algorithms is through reduction to supervised regression with the class $\cF$, which facilitates the use of off-the-shelf supervised learning algorithms and models \citep{foster2020beyond, simchi2021bypassing, foster2020adapting, foster2021instance}. Following \citet{foster2020beyond}, we assume access to an \emph{online regression oracle} \sqalgtext, which is an algorithm for online learning (or, sequential prediction) with the square loss.

We consider the following protocol. At each round $t \in [T]$, the oracle produces an estimator $\fhat_t=f_{\ghat_t}$,
then receives a context-action-reward tuple $(x_t, a_t,
r_t(a_t))$. The goal of the oracle is to accurately predict the reward
as a function of the context and action, and we evaluate its
prediction error via the square loss $(\fhat_t(x_t,a_t)-r_t)^2$. We measure the oracle's cumulative performance through square-loss regret to $\cF$.%
\begin{assumption}[Bounded square-loss regret]
\label{asm:regression_oracle}
The regression oracle \sqalgtext guarantees that for any (potentially adaptively chosen) sequence $\curly*{(x_t, a_t, r_t(a_t))}_{t=1}^T$, 
\begin{align*}
  \sum_{t=1}^T \paren*{\widehat f_t(x_t, a_t) - r_t(a_t)}^2
                   -\inf_{f \in \cF}\sum_{t=1}^T \paren*{f(x_t, a_t) - r_t(a_t)}^2 
  \leq \regsq(T),\notag
\end{align*}
for some (non-data-dependent) function $\regsq(T)$.
\end{assumption}
We let $\Tsq$ denote an upper bound on the time required to (i) query the oracle's estimator $\ghat_t$ with $x_t$ and receive the {vector} $\ghat_t(x_t)\in\bbR^{d}$, and (ii) update the oracle with the example $(x_t, a_t, r_t(a_t))$.
{We let $\Msq$ denote the maximum memory used by the oracle throughout its execution.}

Online regression is a well-studied problem, with computationally efficient algorithms for many models. Basic examples include finite classes $\cF$, where one can attain $\regsq(T) = O(\log \abs*{\cF})$ \citep{vovk1998game}, and linear models ($g(x)= \theta$), where the online Newton step algorithm \citep{hazan2007logarithmic} satisfies \cref{asm:regression_oracle} with $\regsq(T) = O(d \log T)$. More generally, even for classes such as deep neural networks for which provable guarantees may not be available, regression is well-suited to gradient-based methods. We refer to \citet{foster2020beyond, foster2020adapting} for more comprehensive discussion.

\paragraph{Large action spaces: Action optimization oracles}
\label{sec:linear_opt_oracle}
The regression oracle setup in the prequel is identical to that considered in the finite-action setting \citep{foster2020beyond}. In order to develop efficient algorithms for large or infinite action spaces, we assume access to an oracle for linear optimization over actions.
\begin{definition}[Action optimization oracle]
  \label{def:action_oracle}
  An \emph{action optimization oracle} \optalgtext takes as input a context $x \in \cX$, and vector $\theta \in \R^d$ and returns
\begin{align}
    a^\star \ldef \argmax_{a \in \cA} \ang*{\phi(x,a), \theta}. \label{eq:linear_opt_oracle}
\end{align}
\end{definition}
For a single query to the oracle, 
We let $\Topt$ denote a bound on the runtime for a single query to the oracle.
We {let  $\Mopt$ denote the maximum memory used by the oracle throughout its execution.}

The action optimization oracle in \pref{eq:linear_opt_oracle} is widely used throughout the literature on linear bandits \citep{dani2008stochastic, chen2017nearly, cao2019disagreement,katz2020empirical}, and can be implemented in polynomial time for standard combinatorial action spaces. It is a basic computational primitive in the theory of convex optimization, and when $\cA$ is convex, it is equivalent (up to polynomial-time reductions) to other standard primitives such as separation oracles and membership oracles \citep{schrijver1998theory, grotschel2012geometric}. It also equivalent to the well-known Maximum Inner Product Search (MIPS) problem \citep{shrivastava2014asymmetric}, for which sublinear-time hashing based methods are available.

\begin{example}
  \label{ex:combinatorial}
  Let $G=(V,E)$ be a graph, and let $\phi(x,a) \in \crl{0,1}^{\abs{E}}$ represent a matching and $\theta \in \R^{\abs{E}}$ be a vector of edge weights. The problem of finding the maximum-weight matching for a given set of edge weights can be written as a linear optimization problem of the form in \cref{eq:linear_opt_oracle}, and Edmonds' algorithm \citep{edmonds1965paths} can be used to find the maximum-weight matching in $O(\abs{V}^2 \cdot \abs{E})$ time.  
\end{example}
Other combinatorial problems that admit polynomial-time action optimization oracles include the maximum-weight spanning tree problem, the assignment problem, and others \citep{awerbuch2008online, cesa2012combinatorial}.

\paragraph{Action representation}
We define $b_\cA$ as the number of bits used to represent actions in $\cA$, which is always upper bounded by $O(\log \abs{\cA})$ for finite action sets, and by $\wt{O}(d)$ {for actions that can be represented as vectors in $\R^d$}. Tighter bounds are possible with additional structual assumptions. Since representing actions is a minimal assumption, we hide the dependence on $b_\cA$ in big-$\bigoh$ notation for our runtime and memory analysis.

\subsection{Additional Notation}
    We adopt non-asymptotic big-oh notation: For functions
	$f,g:\cZ\to\bbR_{+}$, we write $f=\bigoh(g)$ (resp. $f=\bigom(g)$) if there exists a constant
	$C>0$ such that $f(z)\leq{}Cg(z)$ (resp. $f(z)\geq{}Cg(z)$)
        for all $z\in\cZ$. We write $f=\bigoht(g)$ if
        $f=\bigoh(g\cdot\mathrm{polylog}(T))$, $f=\bigomt(g)$ if $f=\bigom(g/\polylog(T))$.
         We use $\lesssim$ only in informal statements to
highlight salient elements of an inequality.

	For a vector $z\in\bbR^{d}$, we let $\nrm*{z}$ denote the euclidean
	norm. We define $\nrm{z}_{W}^{2}\ldef{}\tri{z,Wz}$ for a positive definite matrix $W \in \R^{d \times d}$.
For an integer $n\in\bbN$, we let $[n]$ denote the set
        $\{1,\dots,n\}$.  
        For a set $\cZ$, we let
        $\Delta(\cZ)$ denote the set of all Radon probability measures
        over $\cZ$. We let $\conv(\cZ)$ denote the set of all finitely
        supported convex combinations of elements in $\cZ$.
        When $\cZ$ is finite, we let $\unif(\cZ)$ denote the uniform distribution over all the
        elements in $\cZ$.
        We let $\indic_{z}\in\Delta(\cZ)$ denote the delta distribution on $z$.
         We use the convention
         $a\wedge{}b=\min\crl{a,b}$ and $a\vee{}b=\max\crl{a,b}$.

\section{Warm-Up: Efficient Algorithms via Uniform Exploration}
\label{sec:warmup}

In this section, we present our first result: an efficient algorithm based on uniform exploration over a representative basis (\spannerGreedy; \cref{alg:greedy}). This algorithm achieves computational efficiency by taking advantage of an online regression oracle, but its regret bound has sub-optimal dependence on $T$.
Beyond being practically useful in its own right, this result serves as a warm-up for \pref{sec:minimax}.

Our algorithm is based on exploration with a \emph{G-optimal design} for the embedding $\phi$, which is a distribution over actions that minimizes a certain notion of worse-case variance \citep{kiefer1960equivalence,atwood1969optimal}.
\begin{definition}[G-optimal design]
    \label{def:g_opt_design}
   Let a set $\cZ \subseteq \R^d$ be given. A distribution $q \in \Delta(\cZ)$ is said to be a G-optimal design with approximation factor $C_{\opt} \geq 1$ if 
   \begin{align*}
       \sup_{z \in \cZ} \nrm{z}_{V(q)^{-1}}^2 \leq C_{\opt} \cdot d,
   \end{align*} 
   where $V(q) \ldef \E_{z \sim q} \brk[\big]{z z^\top}$.
\end{definition}
The following classical result guarantees existence of a G-optimal design.
\begin{lemma}[\citet{kiefer1960equivalence}]
    \label{prop:kiefer_wolfowitz}
    For any compact set $\cZ \subseteq \R^d$, there exists an optimal design with $C_{\opt} = 1$. 
\end{lemma}

\begin{algorithm}[H]
	\caption{\spannerGreedy}
	\label{alg:greedy} 
	\renewcommand{\algorithmicrequire}{\textbf{Input:}}
	\renewcommand{\algorithmicensure}{\textbf{Output:}}
	\newcommand{\algorithmicbreak}{\textbf{break}}
    \newcommand{\BREAK}{\STATE \algorithmicbreak}
	\begin{algorithmic}[1]
		\REQUIRE Exploration parameter $\epsilon \in (0, 1]$, online regression oracle \sqalgtext, action optimization oracle \optalgtext.
		\FOR{$t = 1, 2, \dots, T$}
		\STATE Observe context $x_t$.
		\STATE Receive $\widehat f_t = f_{\widehat g_t}$ from regression oracle \sqalgtext.
\STATE Get $\widehat a_t \gets \argmax_{a \in \cA} \ang*{\phi(x_t, a), \widehat g_t(x_t)}$.
\STATE Call subroutine to compute $C_{\opt}$-approximate optimal design $\des_t \in \Delta(\cA)$ for set $\curly*{\phi(x_t, a)}_{a\in\cA}$.\\
\hfill \algcommentlight{See \cref{alg:barycentric} for efficient solver.}
  \STATE Define $p_t \ldef \eps \cdot \des_t + (1-\epsilon) \cdot \indic_{\wh a_t}$.
  \STATE Sample $a_t \sim p_t$ and observe reward $r_t(a_t)$.
  \STATE Update oracle \sqalgtext with $(x_t, a_t, r_t(a_t))$.
\ENDFOR
	\end{algorithmic}
\end{algorithm}

\cref{alg:greedy} uses optimal design as a basis for exploration: At each round, the learner obtains an estimator $\wh f_t$ from the regression oracle \sqalgtext, then appeals to a subroutine to compute an (approximate) G-optimal design $\des_t \in \Delta(\cA)$ for the action embedding $\crl*{\phi(x_t,a)}_{a\in\cA}$.
Fix an exploration parameter $\eps > 0$, the algorithm then samples an action $a \sim \des_t$ from the optimal design with probability $\veps$ (``exploration''), or plays the greedy action $\wh a_t \ldef \argmax_{a \in \cA} \wh f_t(x_t,a)$ with probability $1-\eps$ (``exploitation''). \pref{alg:greedy} is efficient whenever an approximate optimal design can be computed efficiently, which can be achieved using \pref{alg:barycentric}. We defer a detailed discussion of efficiency for a moment, and first state the main regret bound for the algorithm.

\begin{restatable}{theorem}{thmGreedy}
	\label{thm:greedy}
With a $C_{\opt}$-approximate optimal design subroutine and an appropriate choice for $\veps\in(0,1]$, \cref{alg:greedy}, with probability at least $1-\delta$, enjoys regret
\begin{align*}
     \regcb(T)
      = \bigoh \prn*{ (C_{\opt} \cdot d)^{1/3} T^{2/3} (\regsq(T) + \log(\delta^{-1}))^{1/3}}.
\end{align*}	
In particular, when invoked with \cref{alg:barycentric} {(with $C=2$)} as a subroutine, the algorithm enjoys regret 
\begin{align*}
  \regcb(T)
   = \bigoh\prn*{ d^{2/3}T^{2/3} (\regsq(T) + \log(\delta^{-1}))^{1/3}}.
\end{align*}	
and has per-round runtime $O(\Tsq  + \Topt \cdot d^2 \log d  + d^4 \log d)$ and maximum memory $O( \Msq + \Mopt + d^2)$.
\end{restatable}

\paragraph{Computational efficiency}
The computational efficiency of \cref{alg:greedy} hinges on the ability to efficiently compute an approximate optimal design (or, by convex duality, the John ellipsoid \citep{john1948extremum}) for the set $\crl*{\phi(x_t,a)}_{a\in\cA}$.
All off-the-shelf optimal design solvers that we are aware of require solving quadratic maximization subproblems, which in general cannot be reduced to a linear optimization oracle (\pref{def:action_oracle}). While there are some special cases where efficient solvers exist (e.g., when $\cA$ is a polytope (\citet{cohen2019near} and references therein)), computing an exact optimal design is NP-hard in general \citep{grotschel2012geometric,summa2014largest}.
To overcome this issue, we use the notion of a \emph{barycentric spanner}, which acts as an approximate optimal design and can be computed efficiently using an action optimization oracle.

\begin{definition}[\citet{awerbuch2008online}]
\label{def:barycentric}
Let a compact set $\cZ \subseteq \R^d$ of full dimension be given. For $C\geq{}1$, a subset of points $\cS = \curly*{z_1, \dots, z_d} \subseteq \cZ$ is said to be a $C$-approximate barycentric spanner for $\cZ$ if every point $z \in \cZ$ can be expressed as a weighted combination of points in $\cS$ with coefficients in $[-C, C]$.
\end{definition}

The following result shows that any barycentric spanner yields an approximate optimal design.

\begin{restatable}{lemma}{propBary}
	\label{prop:barycentric_approx_opt_design}
    If $\cS = \crl{z_1, \dots, z_d}$ is a $C$-approximate barycentric spanner for $\cZ \subseteq \R^d$, then $q \ldef \unif(\cS)$ is a $(C^2 \cdot d)$-approximate optimal design.
\end{restatable}
Using an algorithm introduced by \citet{awerbuch2008online}, one can efficiently compute the $C$-approximate barycentric spanner for the set $\crl*{\phi(x,a)}_{a\in\cA}$ using$O(d^2 \log_C d)$ an action optimization oracle; their method is restated stated as \cref{alg:barycentric} in \pref{app:warmup}.

\paragraph{Key features of \pref{alg:greedy}}
While the regret bound for \pref{alg:greedy} scales with $T^{2/3}$, which is not optimal, this result constitutes the first computationally efficient algorithm for contextual bandits with linearly structured actions and general function approximation. Additional features include:
\begin{itemize}
      \item \emph{Simplicity and practicality.} Appealing to uniform exploration makes \pref{alg:greedy} easy to implement and highly practical. In particular, in the case where the action embedding does not depend on the context (i.e., $\phi(x,a) = \phi(a)$) an approximate design can be precomputed and reused, reducing the per-round runtime to $\wt O(\Tsq + \Topt)$ and the maximum memory to $O(\Msq + d)$.
      \item \emph{Lifting optimal design to contextual bandits.} Previous bandit algorithms based on optimal design are limited to the non-contextual setting, and to pure exploration. Our result highlights for the first time that optimal design can be efficiently combined with general function approximation.

\end{itemize}

\paragraph{Proof sketch for \pref{thm:greedy}}
To analyze \pref{alg:greedy}, we follow a recipe introduced by \citet{foster2020beyond,foster2021statistical} based on the \emph{\dectext} (DEC),\footnote{The original definition of the Decision-Estimation Coefficient in \citet{foster2021statistical} uses Hellinger distance rather than squared error. The squared error version we consider here leads to tighter guarantees for bandit problems where the mean rewards serve as a sufficient statistic.} defined as $\dec_\gamma(\cF) \ldef \sup_{\wh{f} \in \conv(\cF), x \in \cX} \dec_{\gamma}(\cF;\wh{f},x)$, where
\begin{align}
	 \dec_{\gamma}(\cF; \wh{f}, x) \ldef \inf_{p \in \Delta(\cA)} \sup_{a^\star \in \cA} \sup_{f^\star \in \cF} 
	 \E_{a \sim p} \bigg[ f^\star(x, a^\star) - f^\star(x, a) - \gamma \cdot \prn{\wh{f}(x,a) - f^\star(x,a)}^2 \bigg] .
\label{eq:dec}
\end{align}
\citet{foster2021statistical} consider a meta-algorithm which, at each round $t$, (i) computes $\fhat_t$ by appealing to a regression oracle, (ii) computes a distribution $p_t\in\Delta(\cA)$ that solves the minimax problem in \pref{eq:dec} with $x_t$ and $\fhat_t$ plugged in, and (iii) chooses the action $a_t$ by sampling from this distribution. One can show (\pref{lm:regret_decomp} in \pref{app:warmup}) that for any $\gamma > 0$, this strategy enjoys the following regret bound:
\begin{align}
	\regcb(T) \approxleq T \cdot \dec_\gamma(\cF) + \gamma \cdot \regsq(T), \label{eq:decomposition}
\end{align}
More generally, if one computes a distribution that does not solve \pref{eq:dec} exactly, but instead certifies an upper bound on the DEC of the form $\dec_\gamma(\cF) \leq \wb \dec_\gamma(\cF)$, the same result holds with $\dec_\gamma(\cF)$ replaced by $\wb \dec_\gamma(\cF)$. \pref{alg:greedy} is a special case of this meta-algorithm, so to bound the regret it suffices to show that the exploration strategy in the algorithm certifies a bound on the DEC.
\begin{restatable}{lemma}{propDecGreedy}
  \label{prop:dec_greedy}
  For any $\gamma\geq{}1$, by choosing $\eps = \sqrt{C_{\opt} \cdot d/ 4\gamma} \wedge 1$, the exploration strategy in \cref{alg:greedy} certifies that $\dec_\gamma(\cF) = \bigoh(\sqrt{C_{\opt} \cdot d/\gamma})$.
\end{restatable}
Using \cref{prop:dec_greedy}, one can upper bound the first term in \cref{eq:decomposition} by $O(T\sqrt{C_{\opt}d/\gamma})$.
The regret bound in \cref{thm:greedy} follows by choosing $\gamma$ to balance the two terms.

\section{Efficient, Near-Optimal Algorithms}
\label{sec:minimax}

In this section we present \spannerIGW (\pref{alg:igw}), an efficient algorithm with $\widetilde O(\sqrt{T})$ regret (\cref{alg:igw}). We provide the algorithm and statistical guarantees in \cref{sec:minimax_statistical}, then discuss computational efficiency in \cref{sec:minimax_computational}.

\subsection{Algorithm and Statistical Guarantees}
\label{sec:minimax_statistical}

Building on the approach in \pref{sec:warmup}, \spannerIGW uses the idea of exploration with an optimal design. However, in order to achieve $\sqrt{T}$ regret, we combine optimal design with the \emph{inverse gap weighting} (IGW) technique. previously used in the finite-action contextual bandit setting \citep{abe1999associative, foster2020beyond}.

Recall that for finite-action contextual bandits, the inverse gap weighting technique works as follows. Given a context $x_t$ and estimator $\fhat_t$ from the regression oracle \sqalgtext, we assign a distribution to actions in $\cA$ via the rule
\begin{align*}
p_t(a) \ldef \frac{1}{\lambda + \gamma\cdot\prn*{\fhat_t(x_t,\ahat_t) - \fhat_t(x_t,a)}},
\end{align*}
where $\ahat_t\ldef\argmax_{a\in\cA}\fhat_t(x_t,a)$ and $\lambda>0$ is chosen such that $\sum_{a}p_t(a)=1$. This strategy certifies that $\dec_{\gamma}(\cF; \fhat_t,x_t)\leq\frac{\abs{\cA}}{\gamma}$, which leads to regret $\bigoh\prn[\big]{\sqrt{\abs{\cA}T\cdot\regsq(T)}}$. While this is essentially optimal for the finite-action setting, the linear dependence on $\abs{\cA}$ makes it unsuitable for the large-action setting we consider.

To lift the IGW strategy to the large-action setting, \pref{alg:igw} combines it with optimal design with respect to a \emph{reweighted embedding}. Let $\fhat\in\cF$ be given. For each action $a \in \cA$, we define a reweighted embedding via
\begin{align}
    \wb{\phi}(x, a) \ldef  \frac{\phi(x,a)}{\sqrt{1 + \eta \paren*{ \wh f(x,\wh a) - \wh f(x, a)}}}, \label{eq:reweighting}
\end{align}
where $\ahat\ldef{}\argmax_{a\in\cA}\fhat(x,a)$ and $\eta >0$ is a reweighting parameter to be tuned later. This reweighting is \emph{action-dependent} since $\wh f(x,a)$ term appears on the denominator. Within \pref{alg:igw}, we compute a new reweighted embedding at each round $t \in [T]$ using $\wh{f}_t = f_{\wh{g}_t}$, the output of the regression oracle \sqalgtext.

\pref{alg:igw} proceeds by computing an optimal design $\des_t \in \Delta(\cA)$ with respect to the reweighted embedding defined in \cref{eq:reweighting}. The algorithm then creates a distribution $q_t \ldef \frac{1}{2} \des_t + \frac{1}{2} \indic_{\wh a_t}$ by mixing the optimal design with a delta mass at the greedy action $\wh a_t$. Finally, in \pref{eq:igw_reweight}, the algorithm computes an augmented version of the inverse gap weighting distribution by reweighting according to $q_t$. This approach certifies the following bound on the \dectext.
\begin{restatable}{lemma}{propDecIGW}
    \label{prop:dec_igw}
For any $\gamma>0$, by setting $\eta = \gamma/ (C_{\opt}\cdot d)$, the exploration strategy used in \cref{alg:igw} certifies that $\dec_\gamma(\cF) = O \prn{{C_{\opt} \cdot d}/{\gamma}}$.
\end{restatable}
This lemma shows that the reweighted IGW strategy enjoys the best of both worlds: By leveraging optimal design, we ensure good coverage for all actions, leading to $\bigoh(d)$ (rather than $\bigoh(\abs{\cA})$) scaling, and by leveraging inverse gap weighting, we avoid excessive exploration, leading $\bigoh(1/\gamma)$ rather than $\bigoh(1/\sqrt{\gamma})$ scaling. Combining this result with \pref{lm:regret_decomp} leads to our main regret bound for \spannerIGW.

\begin{algorithm}[H]
	\caption{\spannerIGW}
	\label{alg:igw} 
	\renewcommand{\algorithmicrequire}{\textbf{Input:}}
	\renewcommand{\algorithmicensure}{\textbf{Output:}}
	\newcommand{\algorithmicbreak}{\textbf{break}}
    \newcommand{\BREAK}{\STATE \algorithmicbreak}
	\begin{algorithmic}[1]
		\REQUIRE Exploration parameter $\gamma > 0$, online regression
                oracle \sqalgtext, action optimization oracle \optalgtext.
                \STATE Define $\eta\ldef{}\frac{\gamma}{C_{\opt}\cdot{}d}$.
\FOR{$t = 1, 2, \dots, T$}
		\STATE Observe context $x_t$.
		\STATE Receive $\widehat f_t = f_{\widehat g_t}$ from regression oracle \sqalgtext.
		\STATE Get $\widehat a_t \gets \argmax_{a \in \cA} \ang*{\phi(x_t, a), \widehat g_t(x_t)}$.
	\STATE Call subroutine to compute $C_{\opt}$-approximate optimal design $\des_t \in \Delta(\cA)$ for reweighted embedding $\curly*{\wb \phi(x_t, a)}_{a\in\cA}$ (\pref{eq:reweighting} with $\wh f = \wh f_t$). 
\hfill \algcommentlight{See \cref{alg:barycentric_reweight} for efficient solver.}
                  \STATE Define $q_t \ldef \frac{1}{2} \des_t + \frac{1}{2} \indic_{\wh a_t}$.
\STATE For each $a \in \supp(q_t)$, define \begin{align}
		    p_t(a) \ldef \frac{q_t(a)}{\lambda + \eta \prn*{ \wh f_t(x_t,\wh a_t) - \wh f_t(x_t, a)}}, \label{eq:igw_reweight}
		\end{align}
                where $\lambda \in [\frac{1}{2}, 1]$ is chosen so that $\sum_{a \in \supp(q_t)} p_t(a) = 1$. 
		 \STATE Sample $a_t \sim p_t$ and observe reward $r_t(a_t)$.
		 \STATE Update \sqalgtext with $(x_t, a_t, r_t(a_t))$.
\ENDFOR
	\end{algorithmic}
\end{algorithm}

\begin{restatable}{theorem}{thmIGW}
  \label{thm:igw}
  Let $\delta\in(0,1)$ be given. With a $C_{\opt}$-approximate optimal design subroutine and an appropriate choice for $\gamma>0$, \cref{alg:igw} ensures that with probability at least $1-\delta$, 
\begin{align*}
\regcb(T) & = \bigoh \prn*{ \,{\sqrt{C_{\opt} \cdot d \,T \, \prn[\big]{\regsq(T)+ \log( \delta^{-1})}}} }.
\end{align*}	
In particular, when invoked with \cref{alg:barycentric_reweight} {(with $C= 2$)} as a subroutine, the algorithm has
\begin{align*}
    \regcb(T) = \bigoh \prn*{d \, \sqrt{T \, \prn[\big]{\regsq(T)+ \log(\delta^{-1})}}},
\end{align*}
and has per-round runtime $O(\Tsq + \prn{\Topt \cdot d^{3} +d^4} \cdot \log^{2} \prn[\big]{ \frac{T}{r}})$ and the maximum memory $O(\Msq + \Mopt + d^2 + d \log \prn[\big]{\frac{T}{r}} )$.
\end{restatable}
\pref{alg:igw} is the first computationally efficient algorithm with $\sqrt{T}$-regret for contextual bandits with
general function approximation and linearly structured action
spaces. In what follows, we show how to leverage the action optimization oracle (\pref{def:action_oracle}) to achieve this efficiency.

\subsection{Computational Efficiency}
\label{sec:minimax_computational}
The computational efficiency of \pref{alg:igw} hinges on the ability to efficiently compute an optimal design. As with \pref{alg:greedy}, we address this issue by appealing to the notion of a barycentric spanner, which serves as an approximate optimal design. However, compared to \pref{alg:greedy}, a substantial additional challenge is that \pref{alg:igw} requires an approximate optimal design for the \emph{reweighted} embeddings. Since the reweighting is action-dependent, the action optimization oracle \optalgtext cannot be directly applied to optimize over the reweighted embeddings, which prevents us from appealing to an out-of-the-box solver (\pref{alg:barycentric}) in the same fashion as the prequel.

\begin{algorithm}[H]
\caption{\textsf{ReweightedSpanner}}
	\label{alg:barycentric_reweight} 
	\renewcommand{\algorithmicrequire}{\textbf{Input:}}
	\renewcommand{\algorithmicensure}{\textbf{Output:}}
	\newcommand{\algorithmicbreak}{\textbf{break}}
	\newcommand{\BREAK}{\STATE \algorithmicbreak}
\begin{algorithmic}[1]
		\REQUIRE Context $x \in \cX$, oracle prediction $\wh g(x) \in \R^d$, 
		{action $\wh a \ldef \argmax_{a \in \cA} \ang{\phi(x,a), \wh g(x)}$,}
		reweighting parameter $\eta > 0$, approximation factor $C > \sqrt{2}$, initial set $\cS = \prn{a_1, \dots, a_d}$ with $\abs{\det(\phi(x, \cS))} \geq r^d$ for $r \in (0, 1)$.
		\WHILE{not break}
		\FOR{$i = 1, \dots, d$}
		\STATE Compute $\theta \in \R^d$ representing linear function $\wb \phi(x,a) \mapsto \det(\wb \phi(x,\cS_i(a)))$, where $\cS_i(a) \ldef (a_1, \ldots,a_{i-1}, a, a_{i+1}, \ldots, a_d)$.
\hfill\algcommentlight{$\wb \phi$ is computed from $f_{\ghat}$, $\ahat$, and $\eta$ via \cref{eq:reweighting}.}
		\STATE Get $a \gets \argmaxIGW(\theta; x, \wh g(x), \eta, r)$. 
                \hfill\algcommentlight{\cref{alg:igw_argmax}.}
		\IF{\mbox{$\abs*{\det( \wb \phi(x, \cS_{i}(a)))} \geq \frac{\sqrt{2}C}{2} \abs*{\det(\wb \phi(x, \cS))}$}}
		\STATE Update $a_i \gets a$.\STATE \textbf{continue} to line 2.
		\ENDIF
		\ENDFOR
		\BREAK
		\ENDWHILE
		\RETURN $C$-approximate barycentric spanner $\cS$.
	\end{algorithmic}
\end{algorithm}

To address the challenges above, we introduce \rwalgtext (\pref{alg:barycentric_reweight}), a barycentric spanner computation algorithm which is tailored to the reweighted embedding $\phibar$. 
{To describe the algorithm, let us introduce some additional notation. For a set $\cS\subseteq\cA$ of $d$ actions, we let $\det(\wb \phi(x, \cS))$ denote the determinant of the $d$-by-$d$ matrix whose columns are $\crl*{\phibar(x,a)}_{a\in\cA}$. \rwalgtext adapts the barycentric spanner computation approach of \citet{awerbuch2008online}, which aims to identify a subset $\cS\subseteq\cA$ with $\abs{\cS}=d$ that approximately maximizes $\abs{\det(\wb \phi(x, \cS))}$.}
The key feature of \rwalgtext is a subroutine, \argmaxIGW (\cref{alg:igw_argmax}), which implements an (approximate) action optimization oracle for the reweighted embedding:
\begin{equation}
  \argmax_{a\in\cA}\tri*{\phibar(x,a),\theta}.\label{eq:reweighted_argmax}
\end{equation}
\argmaxIGW uses line search reduce the problem in \pref{eq:reweighted_argmax} to a sequence of linear optimization problems with respect to the \emph{unweighted} embeddings, each of which can be solved using \optalgtext. This yields the following guarantee for \pref{alg:barycentric_reweight}.

\begin{restatable}{theorem}{thmBarycentricReweight}
	\label{thm:barycentric_reweight}
	Suppose that \pref{alg:barycentric_reweight} is invoked with parameters $\eta > 0$, $r\in(0,1)$, and $C > \sqrt{2}$, and that the initialization set $\cS$ satisfies $\abs{\det(\phi(x,\cS))} \geq r^d$. Then the algorithm returns a $C$-approximate barycentric spanner with respect to the reweighted embedding set $\crl*{\wb \phi(x,a)}_{a\in\cA}$, and does so with 
	$O(\prn{\Topt \cdot d^{3} +d^4} \cdot \log^{2} (e \vee \frac{\eta}{r}))$ runtime and $O(\Mopt + d^2 + d \log (e \vee \frac{\eta}{r}) )$ memory.	
      \end{restatable}
We refer to \pref{app:argmax_reweight} for self-contained analysis of \argmaxIGW.

\begin{algorithm}[]
  \caption{\argmaxIGW}
  \label{alg:igw_argmax} 
  \renewcommand{\algorithmicrequire}{\textbf{Input:}}
  \renewcommand{\algorithmicensure}{\textbf{Output:}}
  \newcommand{\algorithmicbreak}{\textbf{break}}
  \newcommand{\BREAK}{\STATE \algorithmicbreak}
  \begin{algorithmic}[1]
	  \REQUIRE Linear parameter $\theta \in \R^{d}$, context $x \in \cX$, oracle prediction $\wh g(x) \in \R^d$, reweighting parameter $\eta >0$, {initialization} constant $r \in (0,1)$. 
	  \STATE Define $N \ldef \ceil{d \log_{\frac{4}{3}} \prn{\frac{2\eta + 1}{r}}}$.
\STATE Define \mbox{$\cE \ldef  \crl{\prn{\frac{3}{4}}^i}_{i=1}^N \cup \crl{- \prn{\frac{3}{4}}^i}_{i=1}^N$}.
	  \STATE Initialize $\wh \cA = \emptyset$.
	  \FOR{each $\eps \in \cE$}
	  \STATE Compute $\wb \theta \gets 2 \eps \theta + {\eps^2 \eta} \cdot  \wh g(x)$.
\STATE \mbox{Get $a \gets \argmax_{a \in \cA} \ang{\phi(x,a), \wb \theta}$; add $a$ to $\wh \cA$.}
	  \ENDFOR
	  \RETURN $\argmax_{a \in \wh{\cA}} {\ang{\wb \phi(x,a), \theta}}^2$ 
\hfill \algcommentlight{$\wt O(d)$ candidates.}
  \end{algorithmic}
\end{algorithm}

\paragraph{On the initialization requirement}
The runtime for \pref{alg:barycentric_reweight} scales with $\log(r^{-1})$, where $r\in(0,1)$ is such that $\det(\phi(x,\cS))\geq{}r^{d}$ for the initial set $\cS$. In \cref{app:det_assumption}, we provide computationally efficient algorithms for initialization under various assumptions on the action space.

\section{Empirical Results}
\label{sec:experiments}
In this section we investigate the empirical performance of \greedyalg
and \spannerIGW through three experiments. First, we compare the spanner-based
algorithms to state-of-the art finite-action algorithms on a
large-action dataset; this experiment features nonlinear, learned
context embeddings $g\in\cG$.
Next, we study the impact of redundant actions on the statistical
performance of said algorithms.  Finally, we experiment with a large-scale large-action contextual
bandit benchmark, where we find that the spanner-based methods exhibit
excellent performance.

\paragraph{Preliminaries}
{We conduct experiments on three datasets, whose details are summarized in \cref{tab:datasets}.
\oneshotwiki \citep{singh12:wiki-links,oneshotwiki} is a
named-entity recognition task where contexts are text phrases
preceding and following the mention text, and where actions are text
phrases corresponding to the concept names.  
\amazon-3m \citep{Bhatia16} is an extreme multi-label dataset whose
contexts are text phrases corresponding to the title and description
of an item, and whose actions are integers corresponding to item tags.
Actions are embedded into $\R^{d}$ with $d$ specified in \cref{tab:datasets}.
We construct binary rewards for each dataset, and report 90\% bootstrap confidence intervals (CIs) 
of the rewards in the experiments.
We defer other experimental details to \cref{app:exp_basic}.
}
Code to reproduce all results is available at \url{https://github.com/pmineiro/linrepcb}.
\begin{table}[ht]
\caption{Datasets used for experiments.}
\label{tab:datasets}
\centering
\begin{tabular}{c c c c}
\toprule
Dataset & $T$ & $|\cA|$ & $d$ \\
\midrule
\oneshotwiki-\textsf{311} & 622000 & 311 & 50 \\
\oneshotwiki-\textsf{14031} & 2806200 & 14031 & 50 \\
\amazon-\textsf{3m} & 1717899 & 2812281 & 800 \\
\bottomrule
\end{tabular}
\end{table}

\paragraph{Comparison with finite-action baselines}  We compare \greedyalg
and \mainalg with their finite-action counterparts \greedy and
\squarecb \citep{foster2020beyond} on the \oneshotwiki-\textsf{14031}
dataset. We consider \emph{bilinear models} in which regression functions take the form $f(x, a) = \ang{\phi(a), W x}$
where $W$ is a matrix of learned parameters; the \emph{deep
  models} of the form $f(x, a) = \ang{\phi(a), W \wb g(x)}$, where
$\wb g$ is a learned two-layer neural network {and $W$ contains learned parameters as before}.\footnote{Also see \cref{app:exp_basic} for details.}
\pref{tab:oneshotresults} presents our results. We find that \mainalg performs best, and that both spanner-based algorithms
either tie or exceed their finite-action counterparts. In addition, we
find that working with deep models uniformly improves performance for
all methods. We refer to \cref{tab:timings} in \cref{app:exp_timing} for timing
information.

\begin{table}[ht]
\caption{Comparison on \oneshotwiki-\textsf{14031}.  Values are
the average progressive rewards {(confidence intervals)}, scaled by 1000. We include the performance of the best constant 
predictor and supervised learning as a baseline and skyline respectively.}
\label{tab:oneshotresults}
\centering
\begin{tabular}{c c c }
\toprule
Algorithm & \multicolumn{2}{c}{Regression Function} \\
 & Bilinear & Deep \\
\midrule
\textsf{best constant} & \multicolumn{2}{c}{$0.07127$} \\
\greedy & $[ 5.00, 6.27 ]$ & $[ 7.15, 8.52 ]$ \\
\greedyalg & $[ 6.29, 7.08 ]$ & $[ 6.67, 8.30 ]$ \\
\squarecb & $[ 7.57, 8.59 ]$ & $[ 10.4, 11.3 ]$ \\
\mainalg & $[ 8.84, 9.68 ]$ & $[ 11.2, 12.2 ]$ \\
\supervised & $[ 31.2, 31.3 ]$ & $[ 36.7, 36.8 ]$ \\
\bottomrule
\end{tabular}
	\end{table}

\paragraph{Impact of redundancy}
Finite-action contextual bandit algorithms can explore excessively in
the presence of redundant actions. To evaluate performance in the face
of redundancy, we augment
\oneshotwiki-\textsf{311} by duplicating action the final
action. \cref{tab:duplicateaction} displays the performance of
\spannerIGW and its finite-action counterpart, \squarecb, with a
varying number of duplicates. We find that \mainalg is completely
invariant to duplicates (in fact, the algorithm produces numerically
identical output when the random seed is fixed), but \squarecb is
negatively impacted and over-explores the duplicated action.  \greedyalg
and \greedy behave analogously (not shown).

\begin{table}[ht]
\caption{Redundancy study on \oneshotwiki-\textsf{311}.  Values are
the average progressive rewards {(confidence intervals)}, scaled by 100.}
\label{tab:duplicateaction}
\centering
\begin{tabular}{c c c }
\toprule
Duplicates & \mainalg & \squarecb \\
\midrule
0 & $[ 12.6, 13.0 ]$ & $[ 12.2, 12.6 ]$ \\
16 & $[ 12.6, 13.0 ]$ & $[ 12.1, 12.4 ]$ \\
256 & $[ 12.6, 13.0 ]$ & $[ 10.2, 10.6 ]$ \\
1024 & $[ 12.6, 13.0 ]$ & $[ 8.3, 8.6 ]$ \\
\bottomrule
\end{tabular}
\end{table}

\paragraph{Large scale exhibition}
We conduct a large scale experiment using the \amazon-\textsf{3m} dataset.
Following \citet{sen2021top}, we study the top-$k$ setting where $k$ actions are selected 
at each round.
Out of the total number of actions sampled, we let $r$ denote the number of actions sampled for exploration.
We apply \greedyalg for this dataset 
{and consider regression functions similar to the deep models discussed before.}
The setting $(k=1)$ corresponds to running our algorithm
unmodified, and $(k=5, r=3)$ corresponds to selecting 5 actions per round and using 3 exploration slots.
\cref{fig:amazon3m} in
\cref{app:exp_figure} displays the results. For $(k=1)$ the final CI is
$[0.1041, 0.1046]$, and for $(k=5,r=3)$ the final CI is $[0.438, 0.440]$.

{In the setup} with $(k=5, r=3)$, our results are directly comparable to
\citet{sen2021top}, who evaluated a tree-based contextual bandit
method on the same dataset. The best result from \cite{sen2021top}
achieves roughly 0.19 reward with $(k=5, r=3)$,  which we exceed by a factor of 2.
This indicates that our use of embeddings provides favorable inductive bias for
this problem, and underscores the broad utility of our techniques (which leverage embeddings).  For $(k=5,
r=3)$, our inference time on a commodity CPU with batch size 1 is
160ms per example, which is slower than the time of 7.85ms per example
reported in \citet{sen2021top}.

\section{Additional Related Work}
\label{sec:related}

In this section we highlight some relevant lines of research not already discussed.

\paragraph{Efficient general-purpose contextual bandit algorithms}
There is a long line of research on computationally efficient methods for contextual bandits with general function approximation, typically based on reduction to either cost-sensitive classification oracles \citep{langford2007epoch, dudik2011efficient, agarwal2014taming} or regression oracles \citep{foster2018practical, foster2020beyond,simchi2021bypassing}. Most of these works deal with a finite action spaces and have regret scaling with the number of actions, which is necessary without further structural assumptions \citep{agarwal2012contextual}. An exception is the works of \citet{foster2020adapting} and \citet{xu2020upper}, both of which consider the same setting as the present paper. Both of the algorithms in these works require solving subproblems based on maximizing quadratic forms (which is NP-hard in general \citep{sahni1974computationally}), and cannot directly take advantage of the linear optimization oracle we consider.
Also related is the work of \citet{zhang2021feel}, which proposes a posterior sampling-style algorithm for the setting we consider. This algorithm is not fully comparable computationally, as it requires sampling from specific posterior distribution; it is unclear whether this can be achieved in a provably efficient fashion.

\paragraph{Linear contextual bandits} The linear contextual bandit problem is a special case of our setting in which $\gstar(x)=\theta\in\bbR^{d}$ is constant (that is, the reward function only depends on the context through the feature map $\phi$). The most well-studied families of algorithms for this setting are UCB-style algorithms and posterior sampling. With a well-chosen prior and posterior distribution, posterior sampling can be implemented efficiently \citep{agrawal2013thompson}, but it is unclear how to efficiently adapt this approach to accomodate general function approximation. Existing UCB-type algorithms require solving sub-problems based on maximizing quadratic forms, which is NP-hard in general \citep{sahni1974computationally}. One line of research aims to make UCB efficient by using hashing-based methods (MIPS) to approximate the maximum inner product \citep{yang2021linear, jun2017scalable}. These methods have runtime sublinear (but still polynomial) in the number of actions.

\paragraph{Non-contextual linear bandits} For the problem of \emph{non-contextual} linear bandits (with either stochastic or adversarial rewards), there is a long line of research on efficient algorithms that can take advantage of linear optimization oracles \citep{awerbuch2008online, mcmahan2004online, dani2006robbing, dani2008stochastic,bubeck2012towards, hazan2016volumetric, ito2019oracle}; see also work on the closely related problem of combinatorial pure exploration \citep{chen2017nearly, cao2019disagreement,katz2020empirical,wagenmaker2021experimental}.
In general, it is not clear how to lift these techniques to contextual bandits with linearly-structured actions and general function approximation.
{We also mention that optimal design has been applied in the context of linear bandits, but these algorithms are restricted to the non-contextual setting \citep{lattimore2020bandit, lattimore2020learning}, or to pure exploration \citep{soare2014best, fiez2019sequential}.
The only exception we are aware of is \citet{ruan2021linear}, who
extend these developments to linear contextual bandits (i.e., where
$\gstar(x) = \theta$), but critically use that contexts are stochastic. 
}

\paragraph{Other approaches} Another line of research provides efficient contextual bandit methods under specific modeling assumptions on the context space or action space that differ from the ones we consider here. \citet{zhou2020neural, xu2020neural, zhang2021neural, kassraie2022neural} provide generalizations of the UCB algorithm and posterior sampling based on the Neural Tangent Kernel (NTK). These algorithms can be used to learn context embeddings (i.e., $g(x)$) with general function approximation, but only lead to theoretical guarantees under strong RKHS-based assumptions. For large action spaces, these algorithms typically require enumeration over actions. \citet{majzoubi2020efficient} consider a setting with nonparametric action spaces and design an efficient tree-based learner; their guarantees, however, scale exponentially in the dimensionality of action space. 
\citet{sen2021top} provide heuristically-motivated but empirically-effective tree-based algorithms for contextual bandits with large action spaces, with theoretical guarantees when the actions satisfy certain tree-structured properties. Lastly, another empirically-successful approach is the policy gradient method (e.g., \citet{williams1992simple, bhatnagar2009natural, pan2019policy}). On the theoretical side, policy gradient methods do not address the issue of systematic exploration, and---to our knowledge---do not lead to provable guarantees for the setting considered in our paper.

\section{Discussion}
\label{sec:discussion}

We provide the first efficient algorithms for contextual bandits with
continuous, linearly structured action spaces and general-purpose
function approximation.
We highlight some natural directions for future research below.

\begin{itemize}
  \item \textbf{Efficient algorithms for nonlinear action 
      spaces.} Our algorithms take advantage of linearly structured
    action spaces by appealing to optimal design. Can we develop
    computationally efficient methods for contextual bandits with
    nonlinear dependence on the action space?
\item \textbf{Reinforcement learning.} The contextual bandit problem is a special case of the
    reinforcement learning problem with horizon one. Given our
    positive results in the contextual bandit setting, a natural next
    step is to extend our methods to reinforcement learning problems
    with large action/decision spaces. For example,
    \citet{foster2021statistical} build on our computational tools to
    provide efficient algorithms for reinforcement learning with bilinear classes.

\end{itemize}
  Beyond these directions, natural domains in which to extend our
  techniques include pure exploration and off-policy learning with linearly structured actions.

\bibliography{refs}

\newpage
\appendix

\section{Proofs and Supporting Results from \cref{sec:warmup}}
\label{app:warmup}
This section is organized as follows. We provide supporting results in \cref{app:warmup_support}, then give the proof of \cref{thm:greedy} in \cref{app:warmup_greedy}.

\subsection{Supporting Results}
\label{app:warmup_support}

\subsubsection{Barycentric Spanner and Optimal Design}

\cref{alg:barycentric} restates an algorithm of
\citet{awerbuch2008online}, which efficiently computes a barycentric
spanner (\cref{def:barycentric}) given access to a linear optimization
oracle (\cref{def:action_oracle}). 
{Recall that, for a set $\cS\subset\cA$ of $d$ actions, the notation $\det(\wb \phi(x, \cS))$ (resp. $\det(\phi(x,\cS))$) denotes the determinant of the $d$-by-$d$ matrix whose columns are the $\wb \phi$ (resp. $\phi$) embeddings of actions.} 
\begin{algorithm}[H]
    \caption{Approximate Barycentric Spanner \citep{awerbuch2008online}}
    \label{alg:barycentric} 
    \renewcommand{\algorithmicrequire}{\textbf{Input:}}
    \renewcommand{\algorithmicensure}{\textbf{Output:}}
    \newcommand{\algorithmicbreak}{\textbf{break}}
    \newcommand{\BREAK}{\STATE \algorithmicbreak}
    \begin{algorithmic}[1]
        \REQUIRE Context $x \in \cX$ and approximation factor $C > 1$.
\FOR{$i = 1, \dots, d$}
        \STATE Compute $\theta \in \R^d$ representing linear function $\phi(x,a) \mapsto \det(\phi(x,a_1), \ldots, \phi(x, a_{i-1}), \phi(x,a), e_{i+1}, \ldots, e_d)$.

\STATE Get $a_i \gets \argmax_{a \in \cA} \abs{\ang{\phi(x,a), \theta}}$.\ENDFOR
	\STATE Construct $\cS = \prn{a_1, \ldots, a_d}$. \hfill \algcommentlight{Initial set of actions $\cS \subseteq \cA$ such that $\abs{\cS} = d$ and $\abs{\det(\phi(x,\cS))} > 0$.}
        \WHILE{not break}
		\FOR{$i = 1, \dots, d$}
		\STATE Compute $\theta \in \R^d$ representing linear function $\phi(x,a) \mapsto \det(\phi(x, \cS_{i}(a)))$, where $\cS_i(a) \ldef ( a_1, \ldots , a_{i-1}, a, a_{i+1}, \ldots, a_d )$.
\STATE Get $a \gets \argmax_{a \in \cA} \abs{\ang{\phi(x,a), \theta}}$.
		\IF{$\abs{\det( \phi(x,\cS_{i}(a)))} \geq C \abs{\det(\phi(x, \cS))}$}
		\STATE Update $a_i \gets a$. \STATE \textbf{continue} to line 5.
		\ENDIF
		\ENDFOR
		\BREAK
		\ENDWHILE
        \RETURN $C$-approximate barycentric spanner $\cS$.
    \end{algorithmic}
\end{algorithm}

\begin{lemma}[\citet{awerbuch2008online}]
  \label{prop:barycentric_basic}
For any $x \in \cX$, \cref{alg:barycentric} computes a $C$-approximate barycentric spanner for $\crl{\phi(x,a): a \in \cA}$ within $O(d \log_C d)$ iterations of the while-loop.
\end{lemma}

\begin{lemma}
\label{prop:barycentric_basic_computational}	
Fix any constant $C > 1$.
   \cref{alg:barycentric} can be implemented with 
   runtime $O(\Topt \cdot d^2 \log d + d^4 \log d)$ and memory $O(\Mopt + d^2)$.
\end{lemma}

\begin{proof}[\pfref{prop:barycentric_basic_computational}]
We provide the computational complexity analysis starting from the while-loop (line 5-12) in the following.
The computational complexity regarding the first for-loop (line 1-3) can be similarly analyzed.
\begin{itemize}
	\item \emph{Outer loops (lines 5-6).}
From \cref{prop:barycentric_basic}, we know that \cref{alg:barycentric} terminates within $O(d \log d)$  iterations of the while-loop (line 5). 
	It is also clear that the for-loop (line 6) is invoked at most $d$ times.
	 \item \emph{Computational complexity for lines 7-10.}
		 We discuss how to efficiently implement this part using rank-one updates.
We analyze the computational complexity for each line in the following.
\begin{itemize}
	\item \emph{Line 7.}
We discuss how to efficiently compute the linear function $\theta$ through rank-one updates.
	Fix any $Y \in \R^{d}$. 
   Let $\Phi_\cS$ denote the invertible (by construction) matrix whose $k$-th column is $\phi(x,a_k)$ (with $a_k \in \cS$). Using the rank-one update formula for the determinant \citep{meyer2000matrix}, we have 
   \begin{align}
& \det( \phi(x,a_1),\ldots, \phi(x,a_{i-1}), Y, \phi(x,a_{i+1}),\ldots, \phi(x, a_d) ) \nonumber \\
       & = \det \prn[\Big]{\Phi_\cS + \prn[\big]{ Y - \phi(x,a_i)} e_i^\top} \nonumber\\
       & = \det(\Phi_\cS) \cdot \prn[\Big]{1 + e_i^\top \Phi_\cS^{-1} \prn[\big]{Y - \phi(x,a_i)}} \nonumber \\
       & = \ang[\big]{Y, \det(\Phi_\cS) \cdot \prn*{\Phi_\cS^{-1}}^\top e_i}  + \det(\Phi_\cS) \cdot \prn[\big]{ 1 - e_i^\top \Phi_\cS^{-1} \phi(x,a_i) }. \label{eq:det_rank_one} 
   \end{align}
   We first notice that $\det(\Phi_\cS) \cdot \prn[\big]{ 1 - e_i^\top \Phi_\cS^{-1} \phi(x,a_i) } = 0$ since one can take $Y = 0 \in \R ^{d}$. 
   We can then write 
   \begin{align*}
 \det( \phi(x,a_1),\ldots, \phi(x,a_{i-1}), Y, \phi(x,a_{i+1}),\ldots, \phi(x, a_d) ) = \ang{Y, \theta} 
   \end{align*}
   where $\theta = \det(\Phi_\cS) \cdot \prn*{\Phi_\cS^{-1}}^\top e_i$.
   Thus, whenever $\det(\Phi_\cS)$ and  $\Phi_\cS^{-1}$ are known, compute $\theta$ takes  $O(d)$ time. 
   The maximum memory requirement is $O(d^2)$, following from the storage of $\Phi_\cS^{-1}$.
   \item \emph{Line 8.} When $\theta$ is computed, we can compute  $a$ by first compute  $a_+ \ldef \argmax_{a \in \cA} \ang{\phi(x,a), \theta}$ and $a_- \ldef \argmax_{a \in \cA} - \ang{\phi(x,a), \theta}$ and then compare the two. This process takes two oracle calls to \optalgtext, which takes $O(\Topt)$ time.
	   The maximum memory requirement is $O(\Mopt + d)$, following from the memory requirement of \optalgtext and the storage of $\theta$.
\item \emph{Line 9.}
	Once $\theta$ and $\det(\Phi_\cS)$ are computed, checking the updating criteria takes  $O(d)$ time. The maximum memory requirement is $O(d)$, following from the storage of $\phi(x, a)$ and $\theta$.
 \item \emph{Line 10.}
We discuss how to efficiently update $\det(\Phi_\cS)$ and  $\Phi_\cS^{-1}$ through rank-one updates.
	If an update $a_i = a$ is made, we can update the determinant using rank-one update (as in \cref{eq:det_rank_one}) with runtime $O(d)$ and memory  $O(d^2)$; and update the inverse matrix using the Sherman-Morrison rank-one update formula \citep{sherman1950adjustment}, i.e.,
\begin{align*}
       \prn[\Big]{\Phi_\cS + \prn[\big]{ \phi(x,a) - \phi(x,a_i)} e_i^\top}^{-1} = \Phi_\cS^{-1} -
       \frac{\Phi_\cS^{-1}\prn[\big]{ \phi(x,a) - \phi(x,a_i)} e_i^\top \Phi_\cS^{-1}}{1 + e_i \Phi_\cS^{-1} \prn[\big]{ \phi(x,a) - \phi(x,a_i)} },
\end{align*}
which can be implemented in $O(d^2)$ time and memory.
Note that the updated matrix must be invertible by construction.
\end{itemize}
Thus, using rank-one updates, the total runtime adds up to $O(\Topt + d^2)$ and the maximum memory requirement is $O(\Mopt + d^2)$.
We also remark that the initial matrix determinant and inverse can be computed cheaply since the first iteration of the first for-loop (i.e., line 2 with $i=1$) is updated from the identity matrix.
\end{itemize}
   To summarize,
   \cref{alg:barycentric} has runtime $O(\Topt \cdot d^2 \log d + d^4 \log d)$ 
   and uses at most $O(\Mopt + d^2) $ units of memory. 
\end{proof}

The next proposition shows that a barycentric spanner implies an approximate optimal design. The result is well-known (e.g., \citet{hazan2016volumetric}), but we provide a proof here for completeness.

\propBary*

\begin{proof}[\pfref{prop:barycentric_approx_opt_design}]
  Assume without loss of generality that $\cZ \subseteq \R^d$ spans $\R^d$. By \cref{def:barycentric}, we know that for any $z \in \cZ$, we can represent $z$ as a weighted sum of elements in $\cS$ with coefficients in the range $[-C, C]$. Let $\Phi_\cS \in \R^{d \times d}$ be the matrix whose columns are the vectors in $\cS$. For any $z \in \cZ$, we can find $\theta \in [-C,C]^d$ such that $z = \Phi_\cS \theta$.
   Since $\Phi_\cS$ is invertible (by construction), we can write $\theta = \Phi_\cS^{-1} z$, which implies the result via
   \begin{align*}
       C^2 \cdot d \geq \norm{\theta}_2^2 = \norm{z}^2_{(\Phi_\cS \Phi_\cS^\top)^{-1}} = \frac{1}{d}\cdot \norm{z}^2_{V(q)^{-1}}.
   \end{align*}
\end{proof}

\subsubsection{Regret Decomposition}
Fix any $\gamma > 0$.
We consider the following meta algorithm that utilizes the online regression oracle \sqalgtext defined in \cref{asm:regression_oracle}.

For $t = 1, 2, \ldots, T$:
\begin{itemize}
\item Get context $x_t \in \cX$ from the environment and regression function $\wh f_t \in \conv(\cF)$ from the online regression oracle \sqalgtext.
\item Identify the distribution $p_t \in \Delta(\cA)$ that solves the
  minimax problem $\dec_\gamma(\cF; \wh f_t, x_t)$ (defined in
  \cref{eq:dec}) and play action $a_t \sim p_t$.
\item Observe reward $r_t$ and update regression oracle with example
  $(x_t, a_t, r_t)$.
\end{itemize}

The following result bounds the contextual bandit regret for the meta algorithm described above. 
The result is a variant of the regret decomposition based on
the \dectext given in \citet{foster2021statistical}, which generalizes
\citet{foster2020beyond}. The slight
differences in constant terms are due to the difference in reward range.
\begin{lemma}[\citet{foster2020beyond,foster2021statistical}]
    \label{lm:regret_decomp}
   Suppose that \cref{asm:regression_oracle} holds. Then probability at least $1-\delta$, the contextual bandit regret is upper bounded as follows:
   \begin{align*}
       \regcb(T) \leq \dec_{\gamma}(\cF)\cdot{}T + 2 \gamma \cdot \regsq(T) + 64 \gamma \cdot \log (2\delta^{-1}) + \sqrt{8T \log(2 \delta^{-1})}.
   \end{align*} 
\end{lemma}

In general, identifying a distribution that \emph{exactly} solves the
minimax problem corresponding to the DEC may be impractical. However,
if one can identify a distribution that instead certifies an \emph{upper bound}
$\wb{\dec}_{\gamma}(\cF)$ on the \dectext (in the sense that
$\dec_\gamma(\cF) \leq \wb \dec_\gamma(\cF)$), the regret bound in
\cref{lm:regret_decomp} continues to hold with $\dec_\gamma(\cF)$
replaced by $\wb \dec_\gamma(\cF)$.

\subsubsection{\pfref{prop:dec_greedy}}

\propDecGreedy*
\begin{proof}[\pfref{prop:dec_greedy}]
  Fix a context $x \in \cX$. In our setting, where actions are
  linearly structured, we can equivalently write the \dectext
  $\dec_{\gamma}(\cF;\fhat,x)$ as
\begin{align}
 \dec_{\gamma}  (\cG; \wh{g}, x)  \ldef \inf_{p \in \Delta(\cA)} \sup_{a^\star \in \cA} \sup_{g^\star \in \cG} \E_{a \sim p} \Big[\ang[\big]{\phi(x,a^\star) - \phi(x, a), g^\star(x)}  - {\gamma} \cdot \prn[\big]{\ang[\big]{\phi(x,a), g^\star(x) - \widehat g(x)}}^2 \bigg]. \label{eq:dec_linear}
\end{align}
Recall that within our algorithms, $\wh g \in
\conv(\cG)$ is obtained from the estimator $\wh f = f_{\wh g}$ output
by \sqalgtext. We will bound the quantity in
  \pref{eq:dec_linear} uniformly
  for all $x\in\cX$ and $\wh{g}:\cX\to\bbR^d$ with $\nrm{\ghat}\leq{}1$.
  Recall that we assume $\sup_{g \in \cG, x \in \cX} \nrm{g(x)} \leq 1$.
  
Denote $\wh a \ldef \argmax_{a \in \cA} \ang[\big]{\phi(x,a), \wh g(x)}$ and $a^\star \ldef \argmax_{a \in \cA} \ang[\big]{\phi(x,a), g^\star(x)}$. 
For any $\eps \leq 1$,
let $p \ldef \eps \cdot \des + (1 - \eps)\cdot \indic_{\wh a}$, where
$\des \in \Delta(\cA)$ is any $C_{\opt}$-approximate optimal design
for the embedding $\crl*{\phi(x,a)}_{a\in\cA}$. We have the following decomposition.
\begin{align}
    \E_{a \sim p} \sq[\Big]{\ang[\big]{\phi(x,a^\star) - \phi(x, a), g^\star(x)}} & = 
    \E_{a \sim p} \sq[\Big]{\ang[\big]{\phi(x, \wh a) - \phi(x,a), \wh g(x)}} 
    + \E_{a \sim p} \sq[\Big]{\ang[\big]{\phi(x,a), \wh g(x) - g^\star(x)}} \nonumber\\ 
    & \quad +  \prn[\Big]{\ang[\big]{\phi(x,a^\star), g^\star(x)}
    - \ang[\big]{\phi(x,\wh a), \wh g(x)}}.
    \label{eq:dec_greedy_decompostion}
\end{align}
For the first term in \cref{eq:dec_greedy_decompostion}, we have
\begin{align*}
    \E_{a \sim p} \sq[\Big]{\ang[\big]{\phi(x, \wh a) - \phi(x,a), \wh g(x)}} = 
    \eps \cdot \E_{a \sim \des} \sq[\Big]{\ang[\big]{\phi(x, \wh a) - \phi(x,a), \wh g(x)}} 
    \leq 2 \eps \cdot \sup_{x \in \cX,a \in \cA}  \nrm{\phi(x,a)} \cdot \sup_{x \in \cX}  \nrm{\wh g(x)} \leq  2 \eps.
\end{align*}
Next, since
\begin{align*}
    \ang[\big]{\phi(x,a), \wh g(x) - g^\star(x)} \leq \frac{\gamma}{2} \cdot \prn[\big]{\ang[\big]{\phi(x,a), \wh g(x) - g^\star(x)}}^2 + \frac{1}{2 \gamma}
\end{align*}
by AM-GM inequality, we can bound the second term in
\cref{eq:dec_greedy_decompostion} by
\begin{align*}
    \E_{a \sim p} \sq[\Big]{\ang[\big]{\phi(x,a), \wh g(x) - g^\star(x)}} 
    \leq \frac{\gamma}{2} \cdot \E_{a \sim p} \sq[\Big]{\prn[\big]{\ang[\big]{\phi(x,a), \wh g(x) - g^\star(x)}}^2 }
    + \frac{1}{2 \gamma}.
\end{align*}
We now turn our attention to the third term. Observe that since
$\wh{a}$ is optimal for $\ghat$, $\ang[\big]{\phi(x,\wh a), \wh g(x)}
\geq \ang[\big]{\phi(x,a^\star), \wh g(x)}$. As a result, defining
$V(\des) \ldef \E_{a \sim
  \des} \sq{\phi(x,a) \phi(x,a)^\top}$, we have 
\begin{align*}
     \ang[\big]{\phi(x,a^\star), g^\star(x)} - \ang[\big]{\phi(x,\wh a), \wh g(x)} & \leq  \ang[\big]{\phi(x,a^\star) , g^\star(x) - \wh g(x) }\\
     & \leq \nrm[\big]{\phi(x,a^\star)}_{V(\des)^{-1}} \cdot \nrm[\big]{g^\star(x) - \wh g(x)  }_{V(\des)} \\
     & = \frac{1}{2 \gamma \eps}\cdot \nrm[\big]{\phi(x,a^\star)}^2_{V(\des)^{-1}} + \frac{\gamma}{2} \cdot \eps \cdot \E_{a \sim \des} \sq[\Big]{\prn[\big]{\phi(x,a), g^\star(x) - \wh g(x)}^2} \\
     & \leq \frac{C_{\opt} \cdot d}{2 \gamma \eps} + \frac{\gamma}{2}\cdot \E_{a \sim p} \sq[\Big]{\prn[\big]{\phi(x,a), g^\star(x) - \wh g(x)}^2}.
\end{align*}
Here, the third line follows from the AM-GM inequality, and the last
line follows from the ($C_{\opt}$-approximate) optimal design property and the definition of $p$.

Combining these bounds, we have
\begin{align*}
    \dec_\gamma(\cF) = \inf_{p \in \Delta(\cA)} \sup_{a^\star \in \cA} \sup_{g^\star \in \cG}\dec_\gamma(\cG; \wh g, x) \leq 2 \eps + \frac{1}{2 \gamma} + \frac{C_{\opt} \cdot d}{2 \gamma \eps}.
\end{align*}
Since $\gamma \geq 1$, taking $\eps \ldef \sqrt{C_{\opt} \cdot d/ 4\gamma} \wedge 1$ gives
\begin{align*}
    \dec_\gamma (\cF) \leq 2 \sqrt{\frac{C_{\opt} \cdot d}{\gamma}} + \frac{1}{2 \gamma} \leq 3 \sqrt{\frac{C_{\opt} \cdot d}{\gamma}}
\end{align*}
whenever $\veps<1$. On the other hand, when $\eps=1$, this bound holds trivially.
\end{proof}

\subsection{\pfref{thm:greedy}}
\label{app:warmup_greedy}

\thmGreedy*

\begin{proof}[\pfref{thm:greedy}]
	Consider $\gamma \geq 1$.
    Combining \cref{prop:dec_greedy} with \cref{lm:regret_decomp}, we have 
   \begin{align*}
       \regcb(T) \leq 3T \cdot \sqrt{\frac{C_{\opt} \cdot d}{\gamma}}+ 2 \gamma \cdot \regsq(T) + 64 \gamma \cdot  \log (2\delta^{-1}) + \sqrt{8T \log(2 \delta^{-1})}.
   \end{align*} 
   The regret bound in \pref{thm:greedy} immediately follows by
   choosing \[\gamma = \prn*{\frac{3 T \sqrt{C_{\opt}\cdot d}}{2
     \regsq(T) + 64 \log(2\delta^{-1})}}^{2/3} \vee 1.\] In particular, when \cref{alg:barycentric} is
   invoked as a subroutine with parameter $C=2$,
   \cref{prop:barycentric_approx_opt_design} implies that we may take $C_{\opt}\leq{}4d$.

   \noindent\emph{Computational complexity.}
   We now bound the per-round computational complexity of
   \pref{alg:greedy} when \cref{alg:barycentric} is used as a
   subroutine to compute the approximate optimal design. Outside of
   the call to \pref{alg:barycentric}, \cref{alg:greedy} uses $O(1)$ calls
   to \sqalgtext to obtain $\wh g_t(x_t) \in \R^{d}$ and to update $\wh f_t$,
   and uses a single call to \optalgtext to compute $\wh a_t$. With the
   optimal design $\des_t$ returned by \pref{alg:barycentric}
   (represented as a barycentric spanner), sampling from $p_t$ 
   takes at most $O(d)$ time, since $\abs{\supp(p_t)} \leq d+1$.
outside of \pref{alg:barycentric} adds up to $O(\Tsq + \Topt +d)$. In terms of memory, calling \sqalgtext and \optalgtext takes $O(\Msq + \Mopt)$ units, and maintaining the distribution $p_t$ (the barycentric spanner) takes
   $O(d)$ units,
   so the maximum memory (outside of \pref{alg:barycentric}) is $O(\Msq + \Mopt + d)$.
   The stated results follow from combining the computational complexities analyzed in \cref{prop:barycentric_basic_computational}.
\end{proof}

\section{Proofs and Supporting Results from \cref{sec:minimax_statistical}}
\label{app:minimax_statistical}
In this section we provide supporting results concerning
\pref{alg:igw} (\cref{app:igw_support}), and then give the proof of \cref{thm:igw} (\cref{app:igw_thm}).

\subsection{Supporting Results}
\label{app:igw_support}

\begin{lemma}
    \label{prop:lambda_range}
    In \cref{alg:igw} (Eq. \eqref{eq:igw_reweight}),
    there exists a unique choice of $\lambda > 0$ such that $ \sum_{a \in \cA}^{} p_t(a) = 1$, and its value lies in $[\frac{1}{2}, 1]$.
  \end{lemma}
\begin{proof}[\pfref{prop:lambda_range}]
Define $h(\lambda) \ldef \sum_{a \in \supp(q_t) }^{} \frac{q_t(a)}{\lambda + \eta \prn{\wh f_t(x_t, \wh a_t) - \wh f_t(x_t, a)}} $. We first notice that $h(\lambda)$ is continuous and strictly decreasing over  $(0, \infty)$. 
We further have 
\begin{align*}
	h ({1}/{2})
	& \geq \frac{q_t(\wh a_t)}{{1}/{2} + \eta \prn{\wh f_t(x_t, \wh a_t) - \wh f_t(x_t, \wh a_t)}} \geq \frac{1 /2}{1 /2} =1;
\end{align*}
and 
\begin{align*}
        h(1) \leq \sum_{a \in \supp(q_t) } q_t(a) = \frac{1}{2} + \frac{1}{2}\sum_{a \in \supp(\des_t)} \des_t(a) = 1.
\end{align*}
As a result, there exists a unique normalization constant $\lambda^{\star} \in [\frac{1}{2}, 1]$ such that $ h(\lambda^{\star}) = 1 $.
    \end{proof}
\propDecIGW* 
\begin{proof}[\pfref{prop:dec_igw}]
As in the proof of \pref{prop:dec_greedy}, we use the linear structure
of the action space to rewrite the \dectext
  $\dec_{\gamma}(\cF;\fhat,x)$ as
\begin{align*}
 \dec_{\gamma}  (\cG; \wh{g}, x)  \ldef \inf_{p \in \Delta(\cA)} \sup_{a^\star \in \cA} \sup_{g^\star \in \cG} \E_{a \sim p} \Big[\ang[\big]{\phi(x,a^\star) - \phi(x, a), g^\star(x)}  - {\gamma} \cdot \prn[\big]{\ang[\big]{\phi(x,a), g^\star(x) - \widehat g(x)}}^2 \bigg],
\end{align*}
Where $\ghat$ is such that $\fhat=f_{\ghat}$. We will bound the quantity above uniformly
  for all $x\in\cX$ and $\wh{g}:\cX\to\bbR^{d}$.

   Denote $\wh a \ldef \argmax_{a \in \cA} \ang[\big]{\phi(x,a), \wh g(x)}$, $a^\star \ldef \argmax_{a \in \cA} \ang[\big]{\phi(x,a), g^\star(x)}$
   and $\des \in \Delta(\cA)$ be a $C_{\opt}$-approximate optimal
   design with respect to the reweighted embedding $\wb
   \phi(x,\cdot)$). We use the setting $\eta = \frac{\gamma}{C_{\opt} \cdot d}$ throughout the proof.
   Recall that for the sampling distribution in \cref{alg:igw}, we set
   $q \ldef \frac{1}{2} \des + \frac{1}{2} \indic_{\wh a}$ and define
		\begin{align}
		    p(a) = \frac{q(a)}{\lambda + \frac{\gamma}{C_{\opt}\cdot d} \paren*{ \ang[\big]{\phi(x,\wh a) - \phi(x,a), \wh g(x)}}}, \label{eq:igw_reweight_linear}
		\end{align}
                where $\lambda \in [\frac{1}{2}, 1]$ is a
                normalization constant (cf. \cref{prop:lambda_range}).

We decompose the regret of the distribution $p$ in
                \pref{eq:igw_reweight_linear} as
\begin{align}
    \E_{a \sim p} \sq[\Big]{\ang[\big]{\phi(x,a^\star) - \phi(x, a), g^\star(x)}} & = 
    \E_{a \sim p} \sq[\Big]{\ang[\big]{\phi(x, \wh a) - \phi(x,a), \wh g(x)}} 
    + \E_{a \sim p} \sq[\Big]{\ang[\big]{\phi(x,a), \wh g(x) - g^\star(x)}} \nonumber\\ 
    & \quad +  \ang[\big]{\phi(x,a^\star),  g^\star(x) - \wh g(x)}
    + \ang[\big]{\phi(x,a^\star) - \phi(x,\wh a), \wh g(x)}.
    \label{eq:dec_igw_decompostion}
\end{align}
    Writing out the expectation, the first term in \cref{eq:dec_igw_decompostion} is upper bounded as follows.
   \begin{align*}
       \E_{a \sim p} \sq[\Big]{\ang[\big]{\phi(x, \wh a) - \phi(x,a), \wh g(x)}} & = 
       \sum_{a \in \supp(\des) \cup \crl{\wh a}} p(a) \cdot  \ang[\big]{\phi(x, \wh a) - \phi(x,a), \wh g(x)} \\
       & < \sum_{a \in \supp(\des)} \frac{\des(a)/2}{\frac{\gamma}{C_{\opt}\cdot d} \paren*{ \ang[\big]{\phi(x,\wh a) - \phi(x,a), \wh g(x)}}}  \cdot \ang[\big]{\phi(x, \wh a) - \phi(x,a), \wh g(x)}\\ 
       & \leq \frac{C_{\opt}\cdot d}{ 2 \gamma}, 
   \end{align*}
   where we use that $\lambda > 0$ in the second inequality (with the
   convention that $\frac{0}{0} = 0$).
   
   The second term in \cref{eq:dec_igw_decompostion} can be upper
   bounded as in the proof of \cref{prop:dec_greedy}, by applying the
   AM-GM inequality:
   \begin{align*}
       \E_{a \sim p} \sq[\Big]{\ang[\big]{\phi(x,a), g^\star(x) - \wh g(x)}} 
       \leq \frac{\gamma}{2} \cdot \E_{a \sim p} \sq[\Big]{\prn[\big]{\ang[\big]{\phi(x,a), \wh g(x) - g^\star(x)}}^2 }
       + \frac{1}{2 \gamma}.
   \end{align*}
   
The third term in \pref{eq:dec_igw_decompostion} is the most
   involved. To begin, we define $V(p) \ldef \E_{a \sim p}
   \sq{\phi(x,a) \phi(x,a)^\top}$ and apply the following standard bound:
   \begin{align}
        \ang[\big]{\phi(x,a^\star), \wh g(x) - g^\star(x)} 
        & \leq \nrm[\big]{\phi(x,a^\star)}_{V(p)^{-1}} \cdot \nrm[\big]{g^\star(x) - \wh g(x)  }_{V(p)} \nonumber\\
        & \leq \frac{1}{2 \gamma} \cdot \nrm[\big]{\phi(x,a^\star)}^2_{V(p)^{-1}} + \frac{\gamma}{2} \cdot \nrm[\big]{g^\star(x) - \wh g(x)}^2_{V(p)} \nonumber \\
        & = \frac{1}{2 \gamma} \cdot \nrm[\big]{\phi(x,a^\star)}^2_{V(p)^{-1}} + \frac{\gamma}{2}\cdot \E_{a \sim p} \sq[\Big]{\prn[\big]{\phi(x,a), g^\star(x) - \wh g(x)}^2}, \label{eq:dec_igw_1}
   \end{align}
   where the second line follows from the AM-GM inequality. The second
   term in \cref{eq:dec_igw_1} matches the bound we desired, so it
   remains to bound the first term. Let $\dess$ be the following sub-probability measure:
   \begin{align*}
       \dess(a) \ldef
		\frac{\des(a)/2}{\lambda + \frac{\gamma}{C_{\opt}\cdot d} \paren*{ \ang[\big]{\phi(x,\wh a) - \phi(x,a), \wh g(x)}}},
   \end{align*}
   and let $V(\dess) \ldef \E_{a \sim \dess} \sq{\phi(x,a)
     \phi(x,a)^\top}$. We clearly have $V(p) \succeq V(\dess)$ from
   the definition of $p$ (cf. \cref{eq:igw_reweight_linear}). We
   observe that
   \begin{align*}
       V(\dess) & = \sum_{a \in \supp(\dess)} \dess(a) \phi(x,a) \phi(x,a)^\top \\
       & = \frac{1}{2} \cdot \sum_{a \in \supp(\des)} \des(a) \wb \phi(x,a) \wb \phi(x,a)^\top \cdot \frac{1+ \frac{\gamma}{C_{\opt}\cdot d} \paren*{ \ang[\big]{\phi(x,\wh a) - \phi(x,a), \wh g(x)}}}{\lambda + \frac{\gamma}{C_{\opt}\cdot d} \paren*{ \ang[\big]{\phi(x,\wh a) - \phi(x,a), \wh g(x)}}}\\
       & \succeq \frac{1}{2} \cdot \sum_{a \in \supp(\des)} \des(a) \wb \phi(x,a) \wb \phi(x,a)^\top \rdef \frac{1}{2} \wb V(\des), 
   \end{align*}
   where the last line uses that $\lambda \leq 1$. Since $\wb V(\des)$
   is positive-definite by construction, we have that $V(p)^{-1} \preceq V(\dess)^{-1} \preceq 2 \cdot \wb V(\des)^{-1}$. As a result, 
   \begin{align}
      \frac{1}{2 \gamma} \cdot \nrm[\big]{\phi(x,a^\star)}^2_{V(p)^{-1}} & \leq  
      \frac{1}{ \gamma} \cdot \nrm[\big]{\phi(x,a^\star)}^2_{\wb V(\des)^{-1}} \nonumber \\
      & = \frac{1 + \frac{\gamma}{ C_{\opt}\cdot d} \prn[\big]{\ang[\big]{\phi(x,\wh a) - \phi(x,a^\star), \wh g(x)}}}{\gamma} \cdot  \nrm[\big]{\wb \phi(x,a^\star)}^2_{\wb V(\des)^{-1}} \nonumber \\
      & \leq \frac{C_{\opt}\cdot d}{\gamma} + \ang[\big]{\phi(x,\wh a) - \phi(x,a^\star), \wh g(x)}, \label{eq:dec_igw_2}
   \end{align}
   where the last line uses that $\nrm[\big]{\wb
     \phi(x,a^\star)}^2_{\wb V(\des)^{-1}} \leq C_{\opt} \cdot d$,
   since $\des$ is a $C_{\opt}$-approximate optimal design for the set $\crl*{\wb \phi(x,a)}_{a\in\cA}$. Finally, we observe that the second term in \cref{eq:dec_igw_2} is cancelled out by the forth term in \cref{eq:dec_igw_decompostion}.
   
   Summarizing the bounds on the terms in
   \cref{eq:dec_igw_decompostion} leads to:
   \begin{align*}
       \dec_\gamma(\cF) = \inf_{p \in \Delta(\cA)} \sup_{a^\star \in \cA} \sup_{g^\star \in \cG}\dec_\gamma(\cG; \wh g, x) \leq \frac{C_{\opt}\cdot d}{2 \gamma} + \frac{1}{2 \gamma} + \frac{C_{\opt} \cdot d}{ \gamma } \leq \frac{2 \, C_{\opt} \cdot d}{ \gamma }.
   \end{align*}
\end{proof}

\subsection{\pfref{thm:igw}}
\label{app:igw_thm}
\thmIGW*
\begin{proof}
    Combining \cref{prop:dec_igw} with \cref{lm:regret_decomp}, we have 
    \begin{align*}
        \regcb(T) \leq 2T \cdot {\frac{C_{\opt} \cdot d}{\gamma}}+ 2 \gamma \cdot \regsq(T) + 64 \gamma \cdot \log (2\delta^{-1}) + \sqrt{8T \log(2 \delta^{-1})}.
    \end{align*} 
    The theorem follows by choosing 
    \begin{align*}
        \gamma = \prn*{ \frac{C_{\opt}\cdot d \, T}{\regsq(T) + 32 \log(2 \delta^{-1})}}^{1/2}.
    \end{align*}
    In particular, when \cref{alg:barycentric_reweight} is invoked
    as the subroutine with parameter $C=2$, we may take $C_{\opt} = 4
    d$.

   \noindent\emph{Computational complexity.}
    We now discuss the per-round computational complexity of
    \cref{alg:igw}. 
{We analyze a variant of the sampling rule specified in \cref{app:exp_modification} that does not require computation of the normalization constant.}
    Outside of the runtime and memory requirements
    required to compute the barycentric spanner using
    \pref{alg:barycentric_reweight}, which are stated in
    \pref{thm:barycentric_reweight}, \cref{alg:igw} uses $O(1)$ calls
    to the oracle \sqalgtext to obtain  $\wh g_t(x_t) \in \R^{d}$ and update $\wh f_t$,
    and uses a single call to \optalgtext to compute $\wh a_t$. 
With $\wh g_t(x_t)$ and $\wh a_t$, we can compute $\wh f_t(x_t, \wh a_t) - \wh f_t(x_t, a) = \ang{\phi(x_t,\wh a_t) - \phi(x_t,a), \wh g_t(x_t)}$ in $O(d)$ time for any $a \in \cA$; 
   thus, with the optimal design $\des_t$ returned by \pref{alg:barycentric_reweight}
   (represented as a barycentric spanner), 
we can construct the sampling distribution $p_t$ in  $O(d^2)$ time. 
Sampling from $p_t$ takes  $O(d)$ time since $\abs{\supp(p_t)} \leq d+1$. 
This adds up to runtime $O(\Tsq + \Topt + d^2)$.
   In terms of memory, calling \sqalgtext and \optalgtext takes $O(\Msq + \Mopt)$ units, and maintaining the distribution $p_t$ (the barycentric spanner) takes
   $O(d)$ units,
   so the maximum memory (outside of \pref{alg:barycentric_reweight}) is $O(\Msq + \Mopt + d)$.
   The stated results follow from combining the computational
   complexities analyzed in \cref{thm:barycentric_reweight}
, together with the choice of $\gamma$ described above.
\end{proof}

\section{Proofs and Supporting Results from \cref{sec:minimax_computational}}
\label{app:minimax_computational}
This section of the appendix is dedicated to the analysis of \pref{alg:barycentric_reweight}, and organized as follows.
\begin{itemize}
\item First, in \pref{app:argmax_reweight}, we analyze \pref{alg:igw_argmax}, a subroutine of \pref{alg:barycentric_reweight} which implements a linear optimization oracle for the reweighted action set used in the algorithm.
\item Next, in \cref{app:barycentric_reweight}, we prove \cref{thm:barycentric_reweight}, the main theorem concerning the performance of \pref{alg:barycentric_reweight}.
\item Finally, in \cref{app:det_assumption}, we discuss settings in which the initialization step required by \cref{alg:barycentric_reweight} can be performed efficiently.
\end{itemize}
Throughout this section of the appendix, we assume that the context $x\in\cX$ and estimator $\ghat:\cX\to\bbR^{d}$---which are arguments to \pref{alg:barycentric_reweight} and \pref{alg:igw_argmax}---are fixed.

\subsection{Analysis of \pref{alg:igw_argmax} (Linear Optimization Oracle for Reweighted Embeddings)}
\label{app:argmax_reweight}
A first step is to construct an (approximate) argmax oracle (after taking absolute value) with respect to the reweighted embedding $\wb \phi$.
Recall that the goal of \pref{alg:igw_argmax} is to implement a linear optimization oracle for the reweighted embeddings constructed by \cref{alg:barycentric_reweight}. That is, for any $\theta\in\bbR^{d}$,
we would like to compute an action that (approximately) solves
\begin{align*}
    \argmax_{a \in \cA} \abs[\big]{\ang[\big]{\wb \phi(x,a), \theta}} = 
    \argmax_{a \in \cA} {\ang[\big]{\wb \phi(x,a), \theta}}^2. 
\end{align*} 
Define
\begin{equation}
  \label{eq:iota}
  \iota(a) \ldef \ang{\wb \phi(x,a), \theta}^2,\quad\text{and}\quad
  a^\star \ldef \argmax_{a \in \cA} \iota(a).
\end{equation}
The main result of this section, \pref{prop:grid_search}, shows that \cref{alg:igw_argmax} identifies an action that achieves the maximum value in \pref{eq:iota} up to a multiplicative constant.

    \begin{theorem}
      \label{prop:grid_search}
      Fix any $\eta > {}0$, $r\in(0,1)$. Suppose $\zeta \leq \sqrt{\iota(a^{\star})} \leq 1$ for some $\zeta>0$. Then \cref{alg:igw_argmax} identifies an action $\check a$ such that $\sqrt{\iota(\check a)} \geq \frac{\sqrt{2}}{2}\cdot{}\sqrt{\iota(a^\star)}$, and does so with runtime $O(\prn{\Topt+d} \cdot \log (e \vee \frac{\eta}{\zeta}))$ and maximum memory $O(\Mopt + \log (e \vee \frac{\eta}{\zeta}) + d)$.
    \end{theorem}

    \begin{proof}[\pfref{prop:grid_search}]
      Recall from \cref{eq:reweighting} that we have
\begin{align*}
   {\ang[\big]{\wb \phi(x,a), \theta}}^2 = 
    \prn*{ \frac{{\ang[\big]{\phi(x,a), \theta}}}{\sqrt{1 + \eta \ang[\big]{\phi(x,\wh a) - \phi(x,a), \wh g(x)}}} }^2
    =  \frac{{\ang[\big]{\phi(x,a), \theta}}^2}{{1 + \eta \ang[\big]{\phi(x,\wh a) - \phi(x,a), \wh g(x)}}},
\end{align*}
where $\wh a \ldef \argmax_{a \in \cA} \tri*{\phi(x,a),\ghat(x)}$; note that the denominator is at least $1$. To proceed, we use that for any $X\in\bbR$ and $Y^2 > 0$, we have 
\begin{align*}
    \frac{X^2}{Y^2} = \sup_{\eps \in \R} \curly*{2 \eps X - \eps^2 Y^2}. 
\end{align*}
Taking $X = \ang[\big]{\phi(x,a), \theta}$ and $Y^2 = {1 + \eta \ang[\big]{\phi(x,\wh a) - \phi(x,a), \wh g(x)}}$ above, we can write
\begin{align}
     {\ang[\big]{\wb \phi(x,a), \theta}}^2 & = 
     \frac{{\ang[\big]{\phi(x,a), \theta}}^2}{{1 + \eta \ang[\big]{\phi(x,\wh a) - \phi(x,a), \wh g(x)}}} \nonumber \\
& = \sup_{\eps \in \bbR} \crl[\Big]{2 \eps \ang[\big]{\phi(x,a), \theta} - \eps^2 \cdot \prn[\big]{{1 + \eta \ang[\big]{\phi(x,\wh a) - \phi(x,a), \wh g(x)}}}} \label{eq:quotient_linear_o} \\
    & =  \sup_{\eps \in \bbR}\crl[\Big]{\ang[\big]{\phi(x,a), 2 \eps \theta + \eta \eps^2 \wh g(x)} - \eps^2 - \eta \eps^2 \ang[\big]{\phi(x,\wh a), \wh g(x)}} \label{eq:quotient_linear}.
\end{align}
The key property of this representation is that for any \emph{fixed} $\eps \in \bbR$, \cref{eq:quotient_linear} is a linear function of the \emph{unweighted} embedding $\phi$, and hence can be optimized using \optalgtext.
In particular, for any fixed $\eps \in \bbR$, consider the following linear optimization problem, which can be solved by calling \optalgtext:
\begin{align}
     \argmax_{a \in \cA} \crl[\Big]{2 \eps \ang[\big]{\phi(x,a), \theta} - \eps^2 \cdot \prn[\big]{{1 + \eta \ang[\big]{\phi(x,\wh a) - \phi(x,a), \wh g(x)}}}} \rdef \argmax_{a\in \cA} W(a;\eps)\label{eq:lin_opt_fix_eps}.
\end{align}
Define
\begin{align}
\label{eq:opt_eps}
    \eps^\star \ldef \frac{\ang{\phi(x,a^\star), \theta}}{{1 + \eta \ang{\phi(x,\wh a) - \phi(x,a^\star), \wh g(x)}}}.
\end{align}
If $\eps^{\star}$ was known (which is not the case, since $\astar$ is unknown), we could set $\eps = \eps^\star$ in \cref{eq:lin_opt_fix_eps} and compute an action $\wb a \ldef \argmax_{a \in \cA} W(a; \eps^{\star})$ using a single oracle call. We would then have $\iota(\wb a) \geq W(\wb a; \eps^{\star}) \geq W( a^{\star}; \eps^{\star}) = \iota(a^{\star})$, which follows because $\veps^{\star}$ is the maximizer in \pref{eq:quotient_linear_o} for $a=\astar$.

To get around the fact that $\veps^{\star}$ is unknown, \pref{alg:igw_argmax} performs a grid search over possible values of $\veps$. To show that the procedure succeeds, we begin by bounding the range of $\eps^\star$. With some rewriting, we have 
\begin{align*}
    \abs{\eps^\star} = \frac{\sqrt{\iota(a^\star)}}{\sqrt{{1 + \eta \ang{\phi(x,\wh a) - \phi(x,a^\star), \wh g(x)}}}}.
\end{align*}
Since $0 < \zeta \leq \sqrt{\iota(a^{\star})} \leq 1$, we have
        \begin{align*}
             \wb \zeta \ldef
            {\frac{\zeta}{\sqrt{1+2\eta}}} \leq  \abs{\eps^\star}  \leq 1.
        \end{align*}
\pref{alg:igw_argmax} performs a $(3/4)$-multiplicative grid search over the intervals $[\wb \zeta,1]$ and $[-1,- \wb\zeta]$, which uses $2\ceil{\log_{\frac{4}{3}} \prn{\wb \zeta^{-1}}} = O(\log (e \vee \frac{\eta}{\zeta}))$ grid points. It is immediate to that the grid contains $\wb \eps\in\bbR$ such that $\wb \eps \cdot \eps^{\star} > 0$ and
        $\frac{3}{4} \abs*{\eps^\star} \leq \abs*{  \wb \eps} \leq \abs*{ \eps^\star}$. Invoking \cref{lm:reweighted_opt_approx} (stated and proven in the sequel) with $\bar{a} \ldef \argmax_{a\in\cA}W(a; \wb \veps)$ implies that $\iota(\wb a) \geq \frac{1}{2} \iota(a^\star)$. To conclude, recall that \cref{alg:igw_argmax} outputs the maximizer
        \[
          \acheck \ldef \argmax_{a\in\wh{\cA}}\iota(a),
        \]
        where $\wh{\cA}$ is the set of argmax actions encountered by the grid search. Since $\bar{a}\in\wh{\cA}$, we have $\iota(\check a) \geq \iota (\wb a) \geq \frac{1}{2}\iota(a^\star)$ as desired.

	\noindent \emph{Computational complexity.}
        Finally, we bound the computational complexity of \cref{alg:igw_argmax}. \cref{alg:igw_argmax} maintains a grid of $O(\log (e \vee \frac{\eta}{\zeta}))$ points, and hence calls the oracle \optalgtext $O(\log (e \vee \frac{\eta}{\zeta}))$ in total; this takes $O(\Topt \cdot \log (e \vee \frac{\eta}{\zeta}) )$ time.
	Computing the final maximizer from the set $\wh{\cA}$, which contains $O(\log (e \vee \frac{\eta}{\zeta}))$ actions, takes $O(d \log (e \vee \frac{\eta}{\zeta}))$ time (compute each $\ang{\wb \phi(x,a), \wb \theta}^2$ takes $O(d)$ time). Hence, the total runtime of \cref{alg:igw_argmax} adds up to $O(\prn{\Topt+d} \cdot \log (e \vee \frac{\eta}{\zeta}))$.
	The maximum memory requirement is $O(\Mopt + \log (e \vee \frac{\eta}{\zeta}) + d)$, follows from calling \optalgtext, and storing $\cE, \wh \cA$ and other terms such as $\wh g(x), \theta, \wb \theta, \phi(x,a), \wb \phi(x,a)$.
    \end{proof}

\subsubsection{Supporting Results}

\begin{lemma}
    \label{lm:reweighted_opt_approx}
    Let $\eps^{\star}$ be defined as in \cref{eq:opt_eps}.
Suppose $\wb \eps \in \R$ has $  \wb \eps\cdot \eps^\star > 0$ and $\frac{3}{4} \abs*{\eps^\star} \leq \abs*{  \wb \eps} \leq \abs*{ \eps^\star} $. 
Then, if $  \wb a \ldef \argmax_{a\in\cA}W(a;\wb{\veps})$, we have $\iota(\wb  a) \geq \frac{1}{2} \iota(a^\star)$.
\end{lemma}

\begin{proof}[\pfref{lm:reweighted_opt_approx}]
  First observe that using the definition of $\iota(a)$, along with \cref{eq:quotient_linear_o} and \cref{eq:lin_opt_fix_eps}, we have $\iota(\wb a) \geq W(\wb a; \wb \eps) \geq W(a^\star; \wb \eps)$,
  where the second inequality uses that $\abar\ldef \argmax_{a\in\cA}W(a;\vepsbar)$. Since $\wb \eps \cdot \eps^\star > 0$, we have $\sgn(\wb \eps \cdot \ang{\phi(x,a^\star), \theta}) = \sgn( \eps^\star \cdot \ang{\phi(x,a^\star), \theta})$.
  If $\sgn(\wb \eps \cdot \ang{\phi(x,a^\star), \theta}) \geq 0$, then since $\frac{3}{4} \abs*{\eps^\star} \leq \abs*{  \wb \eps} \leq \abs*{ \eps^\star} $, we have  
    \begin{align*}
      W(\astar;\vepsbar)&=2 \wb \eps \cdot \ang[\big]{\phi(x,a^\star), \theta} - \wb \eps^2 \cdot \prn[\big]{{1 + \eta \ang[\big]{\phi(x,\wh a) - \phi(x,a^\star), \wh g(x)}}}\\
      &\geq 
     \frac{3}{2} \eps^\star \cdot \ang[\big]{\phi(x,a^\star), \theta} - (\eps^\star)^2 \cdot \prn[\big]{{1 + \eta \ang[\big]{\phi(x,\wh a) - \phi(x,a^\star), \wh g(x)}}} \\
      &= \frac{1}{2} \frac{\ang{\phi(x,a^\star), \theta}^2}{{1 + \eta \ang{\phi(x,\wh a) - \phi(x,a^\star), \wh g(x)}}} 
      = \frac{1}{2} \iota(a^\star),
    \end{align*}
    where we use that ${{1 + \eta \ang[\big]{\phi(x,\wh a) - \phi(x,a^\star), \wh g(x)}}} \geq 1$ for the first inequality and use the definition of $\eps^\star$ for the second equality.

    On the other hand, when $\sgn(\wb \eps \cdot \ang{\phi(x,a^\star), \theta}) < 0$, we similarly have  
    \begin{align*}
      W(\astar;\vepsbar)&=2 \wb \eps \cdot \ang[\big]{\phi(x,a^\star), \theta} - \wb \eps^2 \cdot \prn[\big]{{1 + \eta \ang[\big]{\phi(x,\wh a) - \phi(x,a^\star), \wh g(x)}}}\\
      & \geq 
     2\eps^\star  \cdot\ang[\big]{\phi(x,a^\star), \theta} - (\eps^\star)^2 \cdot \prn[\big]{{1 + \eta \ang[\big]{\phi(x,\wh a) - \phi(x,a^\star), \wh g(x)}}} 
       = \iota(a^\star).
    \end{align*}
   Summarizing both cases, we have $\iota(\bar a) \geq \frac{1}{2}\iota(a^\star)$.
    \end{proof}

\subsection{\pfref{thm:barycentric_reweight}}
\label{app:barycentric_reweight}

\thmBarycentricReweight*

\begin{proof}[\pfref{thm:barycentric_reweight}]
We begin by examining the range of $\sqrt{\iota(a^{\star})}$ used in \cref{prop:grid_search}.
Note that the linear function $\theta$ passed as an argument to \cref{alg:barycentric_reweight} takes the form $\wb \phi(x,a) \mapsto \det(\wb \phi(x,\cS_i(a)))$, i.e., ${\ang{\wb \phi(x,a), \theta}} =  {\det(\wb \phi(x,\cS_i(a)))}$, where $\cS_i(a) \ldef \prn{a_1, \ldots, a_{i-1}, a, a_{i+1}, \ldots , a_d}$. 
For the upper bound, we have 
\begin{align*}
\abs{\ang{\wb \phi(x,a^{\star}), \theta}} = \abs{\det(\wb \phi(x,\cS_i(a^{\star})))} \leq \prod_{a \in \cS_i(a^{\star})}^{}  \nrm{\wb \phi(x,a)}_2^d \leq \sup_{a \in \cA} \nrm{\phi(x,a)}_2^d \leq 1
\end{align*}
by Hadamard's inequality and the fact that the reweighting appearing in \cref{eq:reweighting} enjoys $\nrm{\wb \phi(x,a)}_2 \leq \nrm{\phi(x,a)}_2$. This shows that $\sqrt{\iota(a^\star)} \leq 1$.
For the lower bound, we first recall that in \cref{alg:barycentric_reweight}, the set $\cS$ is initialized to have $\abs{\det(\phi(x, \cS))} \geq r^{d}$, and thus $\abs{\det(\wb \phi(x, \cS))} \geq \wb r^{d}$, where $\wb r \ldef \frac{r}{\sqrt{1 + 2 \eta}}$ accounts for the reweighting in \cref{eq:reweighting}. Next, we observe that as a consequence of the update rule in \cref{alg:barycentric_reweight}, we are guaranteed that $\abs{\det(\wb \phi(x,\cS))} \geq \wb r^{d}$ across all rounds.
Thus, whenever \cref{alg:igw_argmax} is invoked with the linear function $\theta$ described above, there must exist an action $a\in\cA$ such that $\abs{\ang{\wb \phi(x,a), \theta}} \geq \wb r^{d}$, which implies that $\sqrt{\iota(a^\star)} \geq \wb r^{d}$ and we can take $\zeta \ldef \wb r^{d}$ in \cref{prop:grid_search}.

  We next bound the number of iterations of the while-loop before the algorithm terminates. Let $\wb C \ldef \frac{\sqrt{2}}{2} \cdot C > 1$.  At each iteration (beginning from line 3) of \cref{alg:barycentric_reweight}, one of two outcomes occurs:
    \begin{enumerate}
        \item We find an index $ i \in [d]$ and an action $a \in \cA$ such that $\abs{\det(\wb \phi(x,\cS_i(a)))} > \wb C \abs{\det(\wb \phi(x, \cS))}$, and update $a_i = a$.
\item We conclude that $\sup_{a \in \cA} \max_{i \in [d]} \abs{\det \prn{\wb \phi(x,\cS_i(a))}} \leq C \abs{\det \prn{\wb \phi(x,\cS)}}$ and terminate the algorithm.
    \end{enumerate}
    We observe that 
    (i) the initial set $\cS$ has $\abs{\det(\wb \phi(x, \cS))} \geq \wb r^d$ with $\wb r \ldef \frac{r}{\sqrt{1+2\eta}}$ (as discussed before), 
    (ii) $\sup_{\cS \subseteq \cA, \abs{\cS} = d} \abs{\det(\wb \phi(x,\cS))} \leq 1$ by Hadamard's inequality,
    and (iii)
    each update of $\cS$ increases the (absolute) determinant by a factor of $\wb C$. 
    Thus, fix any $C > \sqrt{2}$, we are guaranteed that \cref{alg:barycentric_reweight} terminates within $O(d \log \prn{ e \vee \frac{\eta}{r}})$ iterations of the while-loop.

    We now discuss the correctness of \pref{alg:barycentric_reweight}, i.e., when terminated, the set $\cS$ is a $C$-approximate barycentric spanner with respect to the reweighted embedding $\wb \phi$.
First, note that by \cref{prop:grid_search}, \pref{alg:igw_argmax} is guaranteed to identify an action $\check a\in\cA$ such that $\abs{\det(\wb \phi(x,\cS_i(\check a)))} > \wb C \abs{\det(\wb \phi(x,\cS))}$ as long as there exists an action $a^\star\in\cA$ such that $\abs{\det(\wb \phi(x, \cS_i(a^\star)))} > C \abs{\det(\wb \phi(x,\cS))}$. As a result, by Observation 2.3 in \citet{awerbuch2008online}, if no update is made and \cref{alg:barycentric_reweight} terminates, we have identified a $C$-approximate barycentric spanner with respect to embedding $\wb \phi$.

\noindent\emph{Computational complexity.}
We provide the computational complexity analysis for \cref{alg:barycentric_reweight} in the following. 
We use $\wb \Phi_\cS$ to denote the matrix whose  $k$-th column is  $\wb \phi(x,a_k)$ with  $a_k \in \cS$.
\begin{itemize}
	\item \emph{Initialization.}
We first notice that, given $\wh g(x) \in \R^d$ and  $\wh a \ldef \argmax_{a \in \cA} \ang{\phi(x,a), \wh g(x)}$, it takes $O(d)$ time to compute  $\wb \phi(x,a)$ for any  $a \in \cA$.
Thus, computing  $\det(\wb \Phi_\cS)$ and  $\wb \Phi_\cS^{-1}$ takes $O(d^2 + d^{\omega}) = O(d^{\omega})$ time, where we use $O(d^{\omega})$ (with $2 \leq \omega \leq 3$) to denote the time of computing matrix determinant/inversion.
The maximum memory requirement is  $O(d^2)$, following from the storage of $\crl{\wb \phi(x,a)}_{a \in \cS}$ and  $\wb \Phi_\cS^{-1}$.
	\item \emph{Outer loops (lines 1-2).}
We have already shown that \cref{alg:barycentric} terminates within $O(d \log \prn{e \vee \frac{\eta}{r}})$  iterations of the while-loop (line 2). 
	It is also clear that the for-loop (line 2) is invoked at most $d$ times.
	 \item \emph{Computational complexity for lines 3-7.}
		 We discuss how to efficiently implement this part using rank-one updates.
We analyze the computational complexity for each line in the following.
The analysis largely follows from the proof of \cref{prop:barycentric_basic_computational}.
\begin{itemize}
	\item \emph{Line 3.}
Using rank-one update of the matrix determinant (as discussed in the proof of \cref{prop:barycentric_basic_computational}), we have  
   \begin{align*}
 \det( \wb \phi(x,a_1),\ldots, \wb \phi(x,a_{i-1}), Y, \wb \phi(x,a_{i+1}),\ldots, \wb \phi(x, a_d) ) = \ang{Y, \theta}, 
   \end{align*}
   where $\theta = \det(\wb \Phi_\cS) \cdot \prn{\wb \Phi_\cS^{-1}}^\top e_i$.
   Thus, whenever $\det(\wb \Phi_\cS)$ and  $\wb \Phi_\cS^{-1}$ are known, compute $\theta$ takes  $O(d)$ time. 
   The maximum memory requirement is $O(d^2)$, following from the storage of $\wb \Phi_\cS^{-1}$.
   \item \emph{Line 4.} When $\theta$ is computed, we can compute  $a$ by 
invoking \argmaxIGW (\cref{alg:igw_argmax}). As discussed in \cref{prop:grid_search}, this step takes runtime $O(\prn{\Topt \cdot d +d^2} \cdot \log (e \vee \frac{\eta}{r}))$ and maximum memory $O(\Mopt + d \log (e \vee \frac{\eta}{r}) + d)$ (by taking $\zeta = \wb r^{d}$ as discussed before).
\item \emph{Line 5.}
	Once $\theta$ and $\det(\wb \Phi_\cS)$ are computed, checking the updating criteria takes  $O(d)$ time. The maximum memory requirement is $O(d)$, following from the storage of $\wb \phi(x, a)$ and $\theta$.
 \item \emph{Line 6.}
	 As discussed in the proof of \cref{prop:barycentric_basic_computational}, if an update $a_i = a$ is made, we can update $\det(\wb \Phi_\cS)$ and $\wb \Phi_\cS^{-1}$ using rank-one updates with $O(d^{2})$ time and memory.
\end{itemize}
Thus, using rank-one updates, the total runtime for line 3-7 adds up to $O(\prn{\Topt \cdot d +d^2} \cdot \log (e \vee \frac{\eta}{r}))$ and maximum memory requirement is $O(\Mopt + d^2+ d \log (e \vee \frac{\eta}{r}) )$.
\end{itemize}
   To summarize,
   \cref{alg:barycentric} has runtime $O(\prn{\Topt \cdot d^{3} +d^4} \cdot \log^{2} (e \vee \frac{\eta}{r}))$
   and uses at most $O(\Mopt + d^2 + d \log (e \vee \frac{\eta}{r}) )$ units of memory. 
    \end{proof}

    \subsection{Efficient Initializations for \cref{alg:barycentric_reweight}}
    \label{app:det_assumption}
    
    In this section we discuss specific settings in which the initialization required by \cref{alg:barycentric_reweight} can be computed efficiently. For the first result, we let $\ball(0,r)\ldef\crl*{x\in\bbR^{d}\mid{}\nrm*{x}_2\leq{}r}$ denote the ball of radius $r$ in $\bbR^{d}$.
\begin{example}
  \label{ex:init1}
  Suppose that there exists $r \in (0,1)$ such that $\ball(0, r) \subseteq \crl{\phi(x,a): a \in \cA}$. Then by choosing $\cS \ldef \crl{r e_1, \dots, r e_d} \subseteq \cA$, we have $\abs{\det \prn{\phi(\cS)}} = r^d$.
\end{example}

    The next example is stronger, and shows that we can efficiently compute a set with large determinant whenever such a set exists.

    \begin{example}
      \label{prop:determinant_initialization}
        Suppose there exists a set $\cS^\star \subseteq \cA$ such that $\abs{\det(\phi(\cS^\star))} \geq \wb r^d$ for some $\wb r > 0$. Then there exists an efficient algorithm that identifies a set $\cS \subseteq \cA$ with $\abs{\det(\phi(\cS))} \geq r^d$ for $r \ldef \frac{\wb r}{8d}$, and does so with 
   runtime $O(\Topt \cdot d^2 \log d + d^4 \log d)$ and memory $O(\Mopt + d^2)$.
      \end{example}
\begin{proof}[Proof for \pref{prop:determinant_initialization}]
The guarantee is achieved by running \cref{alg:barycentric} with $C=2$.
One can show that this strategy achieves the desired approximation guarantee by slightly generalizing the proof of a similar result in \citet{mahabadi2019composable}. In more detail, \citet{mahabadi2019composable} study the problem of identifying a subset $\cS\subseteq\cA$ such that $\abs*{\cS} = k$ and $\det(\Phi_{\cS}^\top \Phi_{\cS})$ is (approximately) maximized, where $\Phi_\cS \in \R^{d \times \abs{\cS}}$ denotes the matrix whose columns are $\phi(x,a)$ for  $a \in \cS$.
We consider the case when $k=d$, and
make the following observations.
 \begin{itemize}
	\item We have 
            $\det(\Phi_{\cS}^\top \Phi_{\cS}) = \paren*{\det(\Phi_{\cS})}^2 = \paren{\det(\phi(x, \cS))}^2$. Thus, maximizing $\det(\Phi_{\cS}^\top \Phi_{\cS})$ is equivalent to maximizing $\abs{\det(\phi(x,\cS))}$.
	 \item 
           The Local Search Algorithm provided in \citet{mahabadi2019composable} (Algorithm 4.1 therein) has the same update and termination condition as \cref{alg:barycentric}. As a result, one can show that the conclusion of their Lemma 4.1 also applies to \cref{alg:barycentric}.
\end{itemize}
    \end{proof}

\section{Other Details for Experiments}
\label{app:experiment}
\subsection{Basic Details}
\label{app:exp_basic}

\paragraph{Datasets}

\oneshotwiki \citep{singh12:wiki-links,oneshotwiki} is a
named-entity recognition task where contexts are text phrases
preceding and following the mention text, and where actions are text
phrases corresponding to the concept names.  We use the python
package \sentencetransformers \citep{reimers-2019-sentence-bert} to
separately embed the text preceding and following the reference into $\R^{768}$,
and then concatenate, resulting in a context embedding in $\R^{1536}$.
We embed the action (mentioned entity) text into $\R^{768}$ and then use
SVD on the collection of embedded actions to reduce the dimensionality
to $\R^{50}$.  The reward function is an indicator function for
whether the action corresponds to the actual entity
mentioned.  \oneshotwiki-311 (resp. \oneshotwiki-14031) is a subset of
this dataset obtained by taking all actions with at least 2000
(resp. 200) examples. 

\amazon-3m \citep{Bhatia16} is an extreme multi-label dataset whose
contexts are text phrases corresponding to the title and description
of an item, and whose actions are integers corresponding to item tags.
We separately embed the title and description phrases using \sentencetransformers,
which leads to a context embedding in $\R^{1536}$.
Following the protocol
used in \citet{sen2021top}, the first 50000 examples are fully supervised,
and subsequent examples have bandit feedback.  We use Hellinger
PCA~\citep{lebret2014word} on the supervised data label cooccurrences to
construct the action embeddings {in $\R^{800}$}. Rewards are binary, and indicate
whether a given item has the chosen tag. Actions that do not occur in the supervised
portion of the dataset cannot be output by the model, but are retained
for evaluation: For example, if during the bandit feedback phase, an example
consists solely of tags that did not occur during the supervised phase,
the algorithm will experience a reward of 0 for every feasible action
on the example. For a typical seed, this results in roughly 890,000 feasible actions for the model. In the $(k=5, r=3)$ setup, we take
the top-$k$ actions as the greedy slate, and then independently
decide whether to explore for each exploration slot {(the bottom $r$ slots)}. For exploration, we
sample from the spanner set without replacement.

\paragraph{Regression functions and oracles}
For bilinear models, regression functions take the form $f(x, a) = \ang{\phi(a), Wx}$, 
where $W$ is a matrix of learned parameters.  For deep
models, regression functions pass the original context through 2 residual
leaky \textsf{ReLU} layers before applying the bilinear layer, $f(x, a) =
\ang{\phi(a), W \wb g(x)}$, where $\wb g$ {is a learned two-layer neural network, and $W$ is a matrix of learned parameters}.
{For experiments with respect to \oneshotwiki datasets, we add a learned bias term for regression functions (same for every action); for experiments with respect to the \amazon-\textsf{3m} dataset, we additionally add an action-dependent bias term that is obtained from the supervised examples.}
The online regression oracle is implemented using \textsf{PyTorch}'s Adam optimizer with log loss (recall that rewards are 0/1).

\paragraph{Hyperparameters}
For each algorithm, we optimize its hyperparameters using random search \citep{bergstra2012random}. Speccifically, hyperparameters are tuned by taking the best of 59 randomly selected configurations for a fixed seed (this seed is not used for evaluation).
A seed determines both dataset shuffling, initialization of
regressor parameters, and random choices made by any action
sampling scheme. 

\paragraph{Evaluation}
We evaluate each algorithm on 32 seeds.
All reported confidence intervals are 90\% bootstrap CIs for the mean.

\subsection{Practical Modification to Sampling Procedure in \spannerIGW}
\label{app:exp_modification}
For experiments with \spannerIGW, we slightly modify the action
sampling distribution so as to avoid computing the normalization
constant $\lambda$. First, we modify the weighted embedding scheme
given in \pref{eq:reweighting} using the following expression:
\begin{align*}
    \wb{\phi}(x_t, a) \ldef \frac{\phi(x_t,a)}{\sqrt{1 + d + \frac{\gamma}{4d} \prn[\big]{ \wh f_t(x_t,\wh a_t) - \wh f_t(x_t, a)}}}. 
\end{align*}
We obtain a $4d$-approximate optimal design for the reweighted
embeddings by first computing a $2$-approximate barycentric spanner
$\cS$, then taking $q_t^{\mathrm{opt}} \ldef \unif(\cS)$. To proceed, let $\wh a_t \ldef \argmax_{a \in \cA} \wh f(x_t,a)$ and $\wb d \ldef
\abs{\cS \cup \crl{\wh a_t}}$. 
We construct the sampling distribution $p_t \in \Delta(\cA)$ as follows:
\begin{itemize}
    \item Set $p_t(a) \ldef \frac{1}{\wb d + \frac{\gamma}{4d}
        \prn[\big]{\wh f_t(x_t,\wh a_t) - \wh f_t(x_t, a)}}$ for each
      $a \in \supp(\cS)$.
\item Assign remaining probability mass to $\wh a_t$.
\end{itemize}
With a small modification to the proof of \cref{prop:dec_igw}, one
can show that this construction certifies that $\dec_\gamma(\cF) = 
\bigoh\prn[\big]{\frac{d^2}{\gamma}}$. Thus, the regret bound
in \cref{thm:igw} holds up to a constant factor.
{Similarly, with a small modification to the proof of \cref{thm:barycentric_reweight}, we can also show that
---with respect to this new embedding---\cref{alg:barycentric_reweight} has $O(\prn{\Topt \cdot d^{3} +d^4} \cdot \log^{2} \prn[\big]{ \frac{d + \gamma /d}{r}})$ runtime and $O(\Mopt + d^2 + d \log \prn[\big]{\frac{d + \gamma /d}{r}} )$ memory.}

\subsection{Timing Information}
\label{app:exp_timing}
\cref{tab:timings} contains timing information the
$\oneshotwiki$-14031 dataset with a bilinear model.  The CPU timings
are most relevant for practical scenarios such as information
retrieval and recommendation systems, while the GPU timings are
relevant for scenarios where simulation is possible.  Timings for
\greedyalg do not include the one-time cost to compute the spanner
set.  Timings for all algorithms use precomputed context and action
embeddings.  For all but algorithms but \mainalg, timings reflect the
major bottleneck of computing the argmax action, since all subsequent
steps take $O(1)$ time with respect to $\abs{\cA}$.  In particular,
\squarecb is implemented using rejection sampling, which does not
require explicit construction of the action distribution.  For
\mainalg, the additional overhead is due to the time required to construct an approximate optimal design for each example.

\begin{table}[H]
    \centering
    \caption{Per-example inference timings for \oneshotwiki-14031.
      CPU timings use batch size $1$ on an Azure
      \textsf{STANDARD\_D4\_V2} machine. GPU timings use batch size
      1024 on an Azure \textsf{STANDARD\_NC6S\_V2} (Nvidia P100-based)
      machine.}
    \label{tab:timings}
    \begin{tabular}{c c c}
    \toprule
         Algorithm & CPU & GPU \\
         \midrule
         \greedy & 2 ms & 10 $\mu$s \\
         \greedyalg & 2 ms & 10 $\mu$s \\
         \squarecb & 2 ms & 10 $\mu$s\\
         \mainalg & 25 ms & 180 $\mu$s \\
    \bottomrule
    \end{tabular}
  \end{table}

  \subsection{Additional Figures}
  \label{app:exp_figure}
  \begin{figure}[H]
\centering\includegraphics[width=.4\linewidth]{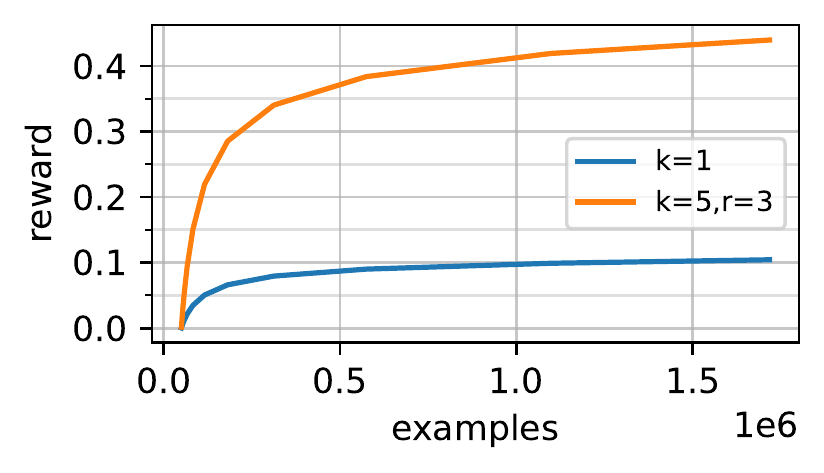}
\caption{\amazon-3m results for \greedyalg.  Confidence intervals are
  rendered, but are but too small to visualize. For $(k=1)$, the final CI is
$[0.1041, 0.1046]$, and for $(k=5,r=3)$, the final CI is $[0.438, 0.440]$.}
\label{fig:amazon3m}
\end{figure}

\end{document}